%% file: main-t-ro-20_revised.tex
\documentclass[journal]{IEEEtran}

\usepackage{amsmath,amssymb,amsfonts,amsthm}
\usepackage{algorithmic}
\usepackage{graphicx}
\usepackage{textcomp}
\usepackage{xcolor}
\usepackage{comment}
\usepackage{multirow}
\usepackage[subpreambles=true]{standalone}
\def\BibTeX{{\rm B\kern-.05em{\sc i\kern-.025em b}\kern-.08em T\kern-.1667em\lower.7ex\hbox{E}\kern-.125emX}}
\usepackage{xspace}
\usepackage[ruled, vlined, longend, linesnumbered]{algorithm2e}
\DontPrintSemicolon
\SetAlFnt{\small}
\usepackage[caption=false, font=footnotesize]{subfig}

\input{macro.tex}

\begin{document}
	
\title{Speeding up Routing Schedules on Aisle-Graphs with Single Access}

\author{
	Francesco Betti Sorbelli,
	Stefano Carpin,~\IEEEmembership{Senior~Member,~IEEE,}
	Federico Corò,
	Sajal K.~Das,~\IEEEmembership{Fellow,~IEEE,}\\
	Alfredo Navarra, and 
	Cristina M.~Pinotti,~\IEEEmembership{Senior~Member,~IEEE}
	\thanks{
Francesco Betti Sorbelli and Sajal Das are with Dept. of Computer Science, Missouri Univ. of Science and Technology, Rolla, MO, USA.
Stefano Carpin is with the Department of Computer Science and Engineering, University of California, Merced, CA, USA.		
Federico Corò is with the Department of Computer Science, Sapienza University of Rome, Italy.
Alfredo Navarra and Cristina M.~Pinotti are with Department of Computer Science and Mathematics, University of Perugia, Italy. \\
 S. Carpin was partially supported USDA-NIFA under award \# 2017-67021-25925. Any opinions, findings, conclusions, or recommendations expressed in this publication are those of the authors and do not necessarily reflect the views of the funding agencies.}
}

\maketitle

\begin{abstract}
In this paper, we study the \problong (\prob), a variant of the orienteering problem for a robot moving in a so-called \emph{single-access aisle-graph}, i.e., 
a graph consisting of a set of rows that can be accessed from one side only. 
Aisle-graphs model, among others, vineyards or warehouses. Each aisle-graph vertex is associated with a reward that a robot obtains when visits the vertex itself. 
As the robot's energy is limited, only a subset of vertices can be visited with a fully charged battery.
The objective is to  maximize the total reward collected by the robot with a battery charge.
We first propose an optimal algorithm that solves \prob in $\mathcal{O}(m^2n^2)$ time for aisle-graphs with a single access consisting of $m$ rows, each with $n$ vertices. 
With the goal of designing faster solutions, we propose four greedy sub-optimal algorithms 
that run in at most $\mathcal{O}(mn\ (m + n))$ time. 
For two of them, we guarantee an approximation ratio of $\frac{1}{2}(1 - \frac{1}{e})$,  
where $e$ is the base of the natural logarithm,
on the total reward by exploiting the well-known submodularity property.
Experimentally, we show that these algorithms collect more than $80\%$ of the optimal reward.
\end{abstract}


\IEEEpeerreviewmaketitle

\maketitle

\section{Introduction}\label{sec:introduction}
\IEEEPARstart{R}{oute} planning is a key problem in robotics and sensor networks.
A classic instance requires that the robot visits multiple Points of Interest (POIs) to perform some tasks, such as to collect data, deliver a package, or perform some repairs.
Depending on the environment in which a robot moves, and based on its capabilities, different trajectories could be calculated.
The decision-making process of selecting a trajectory 
is inherently tied to the energy (fuel) consumption, because a robot must be able to reach its destinations, perform the assigned tasks, and then return to a recharge/refueling station before it runs out of power and remains stranded in the environment.

In this paper, we focus on a robot moving on a specific environment modeled as a graph that belongs to the so-called \emph{aisle-graphs} family recently introduced in~\cite{thayer2018routing}. 
Specifically, we assume the environment is modeled by a set of vertices in rows (\emph{aisles}) 
accessible through one column (\emph{junction line}) at one endpoint of the rows (see Figures~\ref{img:applications} and~\ref{fig:graph-sc}).
This is indeed the topology that can be abstracted from Figure~\ref{img:vineyard} and Figure~\ref{img:warehouse}
when the robots perform complex tasks.
In a vineyard, for example, when the robot moves along a row it has grapevines on its left and right.
For complex tasks, like picking grapes in vineyards~\cite{xiong2018green} (or pruning trees in orchards~\cite{botterill2017robot}),
the robot must act on both the left and right side of the aisles.
\revision{Since the growth of grapes on vines is not homogeneous inside a vineyard, often an indiscriminate full visit is not the most appropriate choice, especially for autonomous energy-constrained vehicles.}
The aisle can then be seen as a two-line aisle: going back and forth, the robot uses the two different lines and
serves always the grapevines on its left (see Figure~\ref{fig:example}). 
Consequently, in this scenario, the robot always enters and exits from the same side of the vineyard.
The same concept can be applied inside a warehouse, like that in Figure~\ref{img:warehouse}.
In fact, in a warehouse aisle, there are shelves  on the left and on the right
from which the robot can pick items.
\revision{Interestingly, a commercial solution which implements a warehouse system with single access for pallets (flat wooden structure) is proposed in~\cite{www-mecalux}.}


\begin{figure}[ht]
	\centering
	\subfloat[Vineyard.]{%
		\includegraphics[width=7.4cm]{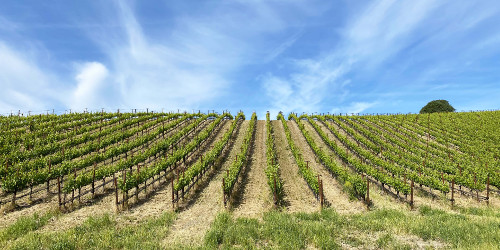}
		\label{img:vineyard}
	}
	\hfill
	\subfloat[Warehouse.]{%
		\includegraphics[width=7.4cm]{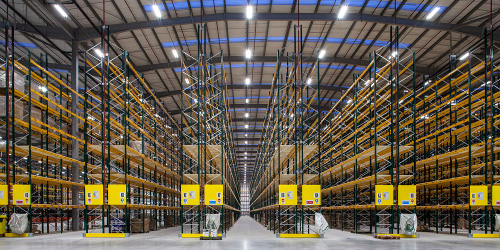}
		\label{img:warehouse}
	}
	\caption{Real examples of aisle-structures with single access.}
	\label{img:applications}
\end{figure}

In the applications we sketched above, the budget
is determined by the maximal distance a robot can travel on a single battery charge, 
the cost of an edge is its length (or the amount of energy spent to traverse it),
and the reward of a vertex is the utility generated by the robot by performing
an assigned task at a certain location associated with the vertex.
Given a single-access aisle-graph with a reward associated to each vertex, 
and assuming an unitary cost for traversing any edge,
our objective is to determine a 
cycle in the graph that can be traversed without exceeding the  preassigned constant budget
and maximizes the total sum of the rewards of the visited vertices.
This is a new variant of the well known \emph{Orienteering} 
Problem (OP)~\cite{golden1987orienteering}, called {\em \problong} (\prob).
As in the original orienteering problem, the reward of visiting a vertex is collected only the first time the vertex is visited by the robot, whereas the cost for traversing an edge
is charged as many times as the robot traverses such an edge. 

\begin{figure}[ht]
	\centering
 	\includegraphics[scale=0.8]{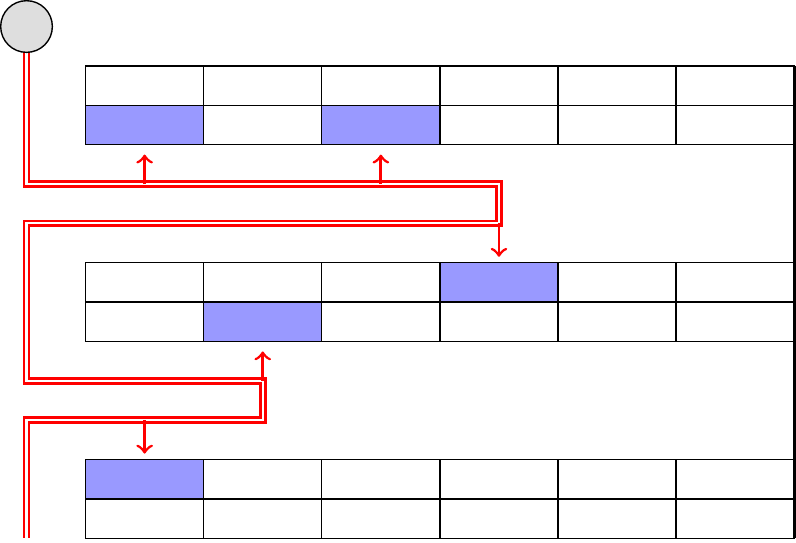}
	\caption{Abstraction of aisle-graph.}
	\label{fig:example}
\end{figure}

As will be discussed below, for general graphs, \revision{the orienteering problem is known to be $\mathit{NP}$-hard~\cite{golden1987orienteering}}, 
thus motivating the design of heuristics, or approximation algorithms. 
However, for the special class of
aisle-graphs with a single access,  the \prob variant of the orienteering problem
is polynomially solvable as we prove 
in Section~\ref{sec:opt-alg}. 
The optimal algorithm \optsc, whose time complexity is a polynomial of grade $4$ in the maximum between the number of aisles and the number of vertices composing a single aisle,
may be not scalable in  real-world applications with tens of thousand aisles, each with hundreds
of vertices.
For example, commercial vineyards often feature
units (called blocks) with more than $50,000$ vines arranged in hundreds of rows  with $200$
or more trees.
The potentially large size of aisle-graphs and the fact that the robot's trajectory has 
to be frequently recomputed due to the reward changes, motivate us 
to search also for simpler and faster 
sub-optimal solutions with guaranteed reward.
Accordingly, we then propose four greedy algorithms, and 
for two of them we prove a constant  factor approximation ratio between 
the reward of their solutions and that of the  optimal solution.


\paragraph{Contributions}
Our results are summarized as follows:
\begin{itemize}
    \item We provide a polynomial time optimal algorithm, called \optsclong (\optsc), for the \problong (\prob).
    
	\item We design two greedy heuristic algorithms, called \gdymelong (\gdyme) and \gdymclong (\gdymc), which are more efficient in terms of time and space complexity with respect to the optimal solution.
	
	\item We provide two additional greedy approximation algorithms faster than the optimal one, but slower than the greedy ones, called \apxmrelong (\apxmre) and \apxmrclong (\apxmrc), which provide a \(\frac{1}{2}(1 - \frac{1}{e})\)-approximation guarantee in terms of the collected reward by exploiting submodularity properties -- with $e$ being the base of the natural logarithm.
	
	\item We evaluate the performance of our algorithms on synthetic and real data, showing that the approximation algorithms collect more than $80\%$ of the optimum reward.
\end{itemize}

This paper extends preliminary results presented 
in~\cite{sorbelli2020optimal} and~\cite{sorbelli2020speeding}. In
particular, the algorithm \optsc discussed herein was 
first introduced in~\cite{sorbelli2020optimal} but is here adapted
to solve a different problem, i.e., \prob.
Its optimality proof is also new and has never appeared before.
Moreover, we provide a much thorough experimental comparison
for all the algorithms on a set of real-world data obtained from a large scale commercial vineyard.

\paragraph{Organization}
The rest of the paper is organized as follows.
Section~\ref{sec:related} reviews the related work on the orienteering problem.
Section~\ref{sec:definition} formally defines the problem and submodularity properties.
Section~\ref{sec:opt-alg} introduces the \optsc algorithm, a polynomial-time  optimal algorithm for \prob, and prove its optimality.
Section~\ref{sec:algorithms} proposes four greedy algorithms in details, while
Section~\ref{sec:approximation} proves a constant factor approximation for the \apxmre and \apxmrc algorithms by exploiting the submodularity property.
Section~\ref{sec:simulations} evaluates the effectiveness of our algorithms on large scale instances through simulation experiments.
Section~\ref{sec:conclusions} offers conclusions.


\section{Related Work}\label{sec:related}
In this section, we provide selected references to previous works on the orienteering problem and submodularity.

\paragraph{Orienteering} 
Orienteering is a classic combinatorial optimization problem defined over graphs with costs associated with edges and rewards associated with vertices.
Introduced in~\cite{golden1987orienteering}, the orienteering problem is also known in the literature as the \emph{bank robber} problem~\cite{awerbuch1998new} and a few other names~\cite{Aghezzaf2016}. 
Orienteering  was shown to be $\mathit{NP}$-hard in~\cite{golden1987orienteering} and $\mathit{APX}$-hard in~\cite{blum2007approximation}.
Numerous variants of the orienteering problem have been proposed in the literature. 
For example, the start vertex may be fixed (rooted orienteering problem) or not (unrooted orienteering problem), edges can be directed or undirected, vertices in the graph may be placed on a plane, or there may be more than one robot collecting rewards. 
The reader is referred to~\cite{gunawan2016orienteering} for a survey of different formulations.
Due to its intrinsic computational complexity, numerous heuristic solutions have been proposed in literature~\cite{gunawan2016orienteering}. 
Exact solutions using branch and bound techniques have been proposed as well~\cite{fischetti1998solving}, but do not scale to problem instances with tens of thousands of vertices like those motivating this work.
A different line of research has aimed at developing approximation algorithms for orienteering, often providing specialized results applicable only to restricted versions of the problem. 
For the rooted version of the problem, the first constant-factor approximation was given in~\cite{Blum2003}, where the authors give a $4$-approximation. 
This result was later improved in~\cite{Chekuri2012}, where for the case of directed graphs the authors give a $\mathcal{O}(2+\varepsilon)$-approximation with a time complexity of $n^{\mathcal{O}(\frac{1}{\varepsilon^2})}$ (where $n$ is the number of vertices in the graph), and for the case of undirected graphs, the authors provide an $\mathcal{O}(\log^2 \opt)$-approximation, where $\opt$ is the number of vertices in an optimal solution.
For the case where the graph is planar and fully connected, a $(1+\varepsilon)$-approximation algorithm was proposed in~\cite{Chen2006}.

\paragraph{Orienteering Applications in Robotics} 
Orienteering has found various applications in robotics. 
These include precision irrigation~\cite{CarpinTASE2020} and robotic harvesting~\cite{Mann2015}. 
The orienteering framework has also been used to study surveillance problems~\cite{ThakurIROS2013}, monitoring~\cite{RusTRO2016}, and information gathering~\cite{Geoff2019}.
Stochastic extensions of the orienteering problem with random edge costs have also been used to study \emph{survival} problems where multiple robots are navigating a dangerous environment and the objective is to visit locations while ensuring a minimal probability of success~\cite{PavoneTOS}.

\paragraph{Orienteering Based on Submodularity}
There exist other works where the orienteering problem is approached observing that the collected rewards may follow the well known {\em submodularity} property.
A function is {\em submodular} if its marginal increment (when a single element is added to the input set) decreases as the size of the input set increases~\cite{nemhauser1978analysis}.

The aforementioned paper~\cite{PavoneTOS} as well as~\cite{jorgensen2017matroid}
studied a multi-robot coordination problem modeled as an orienteering problem  for general graphs leveraging on the submodularity property.
The authors consider a team of robots tasked with visiting sites in a risky environment and subject to team-based operational constraints.
Such a problem is modeled using a graph and a matroid for capturing the team-based operational constraints.
The formulated Matroid Team Surviving Orienteers (MTSO) problem has a submodular structure, and hence polynomial-time algorithms with a guaranteed solution are developed.
In~\cite{xu2020approximation, xu2020approximation2}, approximation algorithms are proposed for the Team Orienteering Problem in the context of mobile charging of rechargeable sensors.
This problem has many potential applications related to the Internet of Things (IoT) and smart cities, such as dispatching energy-constrained mobile chargers to charge as many energy-critical sensors as possible to prolong the network lifetime.
In~\cite{roberts2017submodular}, the submodularity property is used to design a trajectory planning algorithm for an aerial 3D scanning employing a drone.
The authors leveraged submodularity to develop a computationally efficient method for generating scanning trajectories, that reasons jointly about coverage rewards and travel costs.
However, none of the above works addresses the constrained aisle-graph structure.

In~\cite{ghuge2020quasi}, an approximation algorithm is proposed that runs in quasi-polynomial time for the submodular tree orienteering problem.
\revision{Differently from us, they work on directed graphs.
Moreover, their approximation ratio is $\frac{\log k}{\log \log k}$, where $k \le |V|$ is the number of vertices in an optimal solution, while in our work we guarantee a constant approximation ratio of $\frac{1}{2}(1 - \frac{1}{e})$.}
In~\cite{shi2020robust}, an interesting new problem called Robust Multiple-Path Orienteering Problem, in which the main goal is to construct a set of paths for robots guaranteeing robustness in case of malicious attacks, is proposed.
The authors provided an approximation algorithm based on submodularity.

\paragraph{Orienteering on Aisle-graph Family}
The orienteering problem for aisle-graphs was first investigated in~\cite{thayer2018routing} which considered variants on the graph topology and robot capabilities along with experimental results. 
The authors in~\cite{thayer2018routing}, working on aisle-graphs with two accesses (two junction lines, one at each endpoint of the aisles), proposed two greedy heuristics, \gfrlong (\gfr) and \gprlong (\gpr), which respectively select a subset of full or partial rows to be traversed.
At each selection, the robot computes the budget required to collect rewards from its current position and prefers full/partial rows with maximum reward per unit of budget.
The time complexities of \gfr and \gpr are $\mathcal{O}(m^2)$ and $\mathcal{O}(m^2n)$, respectively, where $m$ is the number of aisles and $n$ is the number of vertices composing each aisle.
The authors in~\cite{sorbelli2020optimal} developed polynomial-time algorithms to improve some of the problems faced in~\cite{thayer2018routing}. 
In particular, they designed an optimal algorithm, called \ofrilong (\ofri), that improves on \gfr by determining the optimal solution for the full-row policy, whose time complexity is $\mathcal{O}(m\cdot \max\{n,\log m\})$.
Moreover, they proposed \hgclong (\hgc), which slightly improves on \gpr at a higher time complexity.
In~\cite{thayer2018arouting}, authors studied the problem of routing multiple robots within a vineyard, where movement is limited when a row is entered, for the application of precision irrigation, proposing three algorithms combining \gfr and \gpr either in series or in parallel.

\revision{In a seminal paper in~\cite{ratliff1983order}, aisle-graphs have been investigated for solving the original version of the well known Traveling Salesman Problem (TSP).
The OP is a particular variant of TSP where profits/rewards are introduced.
In other words, the goal of TSP is only to find the minimum cost cycle in the area such that all the points are visited exactly once.
Despite TSP is $\mathit{NP}$-hard~\cite{applegate2006traveling} on general graphs, the authors in~\cite{ratliff1983order} have found an optimal polynomial algorithm (linear in the number of aisles) for aisle-graphs with two accesses based on an efficient dynamic programming approach.
Further, this result has been generalized by the authors in~\cite{cambazard2018fixed} providing an optimal pseudo-polynomial algorithm with time complexity $\mathcal{O}(nh7^h)$ where $h \le n$ denotes the number of horizontal accesses in the area.
}

Finally, in~\cite{bettisorbelli2019automated}, aisle-graphs are used to model a warehouse where an automated picking system is implemented. 
The distance traveled for collecting items belonging to a customer order by a robot, forced to follow the aisles, is compared with the distance traversed by a drone that accomplishes the same goal flying freely inside the warehouse.

\section{Problem Definition}\label{sec:definition}
Consider an undirected {\em aisle-graph} $A(m, n) = (V, E)$, 
where $m$ denotes the number of rows (aisles), $n$ denotes the number of columns (number of vertices composing each aisle), while rows are all connected only via the first column.
Formally, the set of vertices is defined as $V = \{ v_{i,j} | 1 \le i \le m, 1 \le j \le n \}$. 
The set of edges $E$ is defined as follows:
\begin{itemize}
	\item Each vertex $v_{i,j}$ with $1 \le i \le m$ and $1 < j < n$ has two edges, 
	one toward $v_{i,j-1}$ and the other toward $v_{i,j+1}$;
	\item each vertex $v_{i,1}$ with $1 < i < m$ has three edges: one toward $v_{i-1,1}$, 
	one toward $v_{i+1, 1}$, and one toward $v_{i,2}$.
	Accordingly, edges connected to the corner vertices $v_{1,1}$ are $v_{m,1}$ are well defined.
\end{itemize}

\begin{problem}[{\bf \problong (\prob)}]
	Let $A(m, n)$ be an aisle-graph; $v_{1, 1} \in V$ be the home vertex; $\mathcal{R} : V \rightarrow \R_{\ge 0}$ be a reward function where $\mathcal{R}(v_{i,j})=0$ if $j=1$; let $\alpha$ be a positive constant cost (in terms of the budget required) to traverse any edge; \revision{let $B>0$ be the budget, i.e., the maximum distance that can be traveled along edges.}
	The \prob aims to find a cycle of maximum reward starting at $v_{1, 1}$
	of cost no greater than $B$.
\end{problem}

Without loss of generality, we assume $\alpha=1$. 
Figure~\ref{fig:graph-sc} shows an aisle-graph $A(4,5)$. 
The vertices in the first column (for $j=1$) do not provide any reward, and they are used simply to connect the $m$ rows.
In what follows, by $r_i$ and $c_j$ we denote the $i$-th row and the $j$-th column of $A$, respectively,
and by $ R = \{ r_1, \ldots, r_{m} \}$ and $ C = \{ c_1, \ldots, c_{n} \}$ the set of rows and columns of $A$, respectively.

\begin{figure}[ht]
	\centering
 	\includegraphics[scale=0.90]{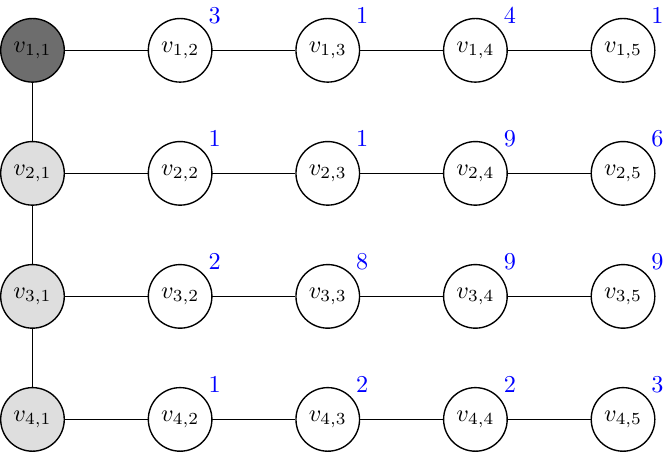}	
	\caption{Aisle-graph $A(4, 5)$; the interconnecting vertices are in dark, the home vertex is the grayest, \revision{the edges costs are unitary}, and the rewards $\mathcal{R}(v_{i,j})$ are in blue (note that interconnecting vertices have no reward).}
	\label{fig:graph-sc}
\end{figure}

\subsection{Submodularity}
In this section, we revise the fundamental definition of submodularity that will be used to derive some of our results.

\begin{definition}[Submodularity~\cite{nemhauser1978analysis}]
    Given a finite set \(V = \{v_1, \ldots, v_n \}\), a set function \(f : 2^V \rightarrow \mathbb{R}\) is submodular if for any \(X \subseteq Y \subseteq V\) and \(v \in V\setminus Y\),
    \begin{equation}\label{eq:sumbodularity def 1}
        f (X \cup v) - f (X) \ge f (Y \cup v) - f (Y)
    \end{equation}
    or equivalently, for each $X, Y \subseteq V$
    \begin{equation}\label{eq:sumbodularity def 2}
        f(X) + f(Y) \ge f(X \cup Y) + f(X \cap Y)
    \end{equation}
\end{definition}

A well-studied class of problems aims to select a \(k\)-element subset that maximizes a monotone submodular objective \(f\) classified with respect to the cost function.
The simplest subset selection problem with cardinality constraint \(|X| \le k\) is $\mathit{NP}$-hard~\cite{cornnejols1977location}. 
The greedy algorithm, which iteratively selects one element with the largest marginal gain, 
can achieve \((1 - \frac{1}{e})\)-approximation guarantee~\cite{nemhauser1978analysis}.  
The best known guarantees are \(\frac{1}{\kappa}(1 - e^{-\kappa})\) due to~\cite{conforti1984submodular}, and
(\(1 - \frac{\kappa}{e}\)) due to~\cite{sviridenko2013tight} for a cardinality constraint,
where \(0 \le \kappa \le 1\) is called {\em curvature}~\cite{conforti1984submodular} that substantially  characterizes how close a monotone submodular function \(f\) is to the modularity.
A generic function \(f\) is {\em modular}
if and only if \(\kappa = 0\).
A function $f: 2^N \rightarrow \mathbb{R}$  is monotone {\em modular} (also called ``additive'' or ``linear'') if \(f(X) = \sum_{j \in X}{w(j)}\) for some $w: N \rightarrow \mathbb{R}_{\ge 0}$.

For the class of problems aiming to select a subset that maximizes a monotone submodular objective \(f\) 
under the cost constraint \(\mathcal{C}(X) \le B\), where \(\mathcal{C}\) is a linear function, the greedy rule selecting the element with the largest marginal gain on \( f\) leads to an unbounded approximation ratio~\cite{khuller1999budgeted}. 
Instead, the generalized greedy rule that iteratively selects the element with the largest ratio of the marginal 
gain on \(f\) and \(\mathcal{C}\), and outputs the better of the found subset and the best single element achieves \(\frac{1}{2}(1 - \frac{1}{e})\)-approximation guarantee~\cite{krause2005note}. 

\subsection{Reward and Cost Modularity}
\prob aims to select a subset $S$ of vertices $V$ in $A$ 
that maximizes the reward under a constraint on the traveling cost to reach $X$. 
To formalize \prob, recall that the orienteering problem with a budget constraint can be stated as:
\begin{equation}\label{eq:submodular-orienteering-subset}
S = \argmax_{X \subseteq V} {\{ \mathcal{R}(X)\; |\; \mathcal{C}(X) \le B \}}
\end{equation}
where $\mathcal{R}$ is the modular monotone reward function to maximize, $\mathcal{C}$ is the modular monotone cost function subject to the constraint to be limited by $B$,  the  budget given as input.

Due to the specific structure of the aisle-graph $A$, 
given a subset of vertices $X$, the cycle of minimum cost that visits all the vertices in $X$ 
traverses the 
subgraph $T_X$ of $A$ containing $v_s=v_{1,1}$ that consists of the vertices in $X$ along with the minimum set of vertices that make it a connected subgraph of $A$ (see, e.g., Figure~\ref{fig:TX}). Such a subgraph is uniquely defined because there is only one simple path from any two vertices of $A$, i.e., $T_X$ is a tree.
\begin{figure}[htbp]
    \centering
	\includegraphics[scale=1.4]{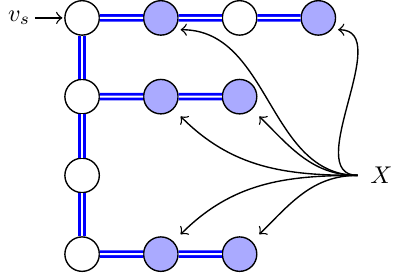}
    \caption{
	The selected subset of vertices $X \subseteq V$ (in blue) and the 
	subgraph $T_X$ that connects $X$ (whole subgraph rooted in $v_s$ with double edges in blue).}
	\label{fig:TX}
\end{figure}

Given $X$, the cycle of minimum cost to visit $X$ starts from $v_{1,1}$, visits the vertices in $T_X$ scanning the rows of $A$ in order from the top to the bottom.
Each row is visited back-and-forth up to its furthest vertex in $X$. 
After visiting the last row \imax that has vertices in $X$, the cycle is completed by vertically traversing the first column from $v_{i_{\max},1}$ to $v_{1,1}$.
The cost of such a cycle is given by the sum of the costs of each traversed edge multiplied by the number of times it is traversed, that is, $\mathcal{C}(X)=2(|T_X|-1)$.
\begin{figure}[htbp]
	\centering
	\subfloat[$\mathcal{C}(X \cup v_{5,4})$.]{%
		\includegraphics[scale=1.4]{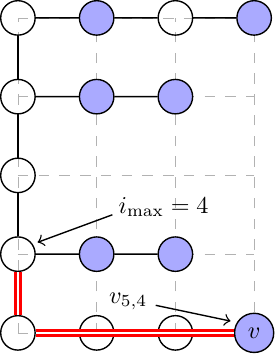}
		\label{fig:tree-a}
	}
	\hfill
	\subfloat[$\mathcal{C}(X \cup v_{3,4})$.]{%
		\includegraphics[scale=1.4]{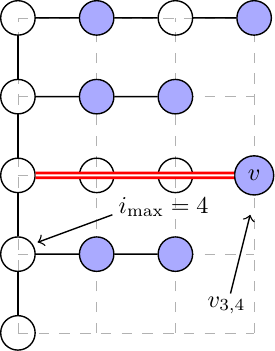}
		\label{fig:tree-b}
	}
	\hfill
	\subfloat[$\mathcal{C}(X \cup v_{4,4})$.]{%
		\includegraphics[scale=1.4]{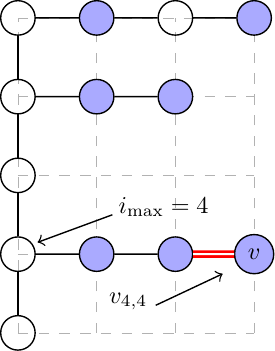}
		\label{fig:tree-c}
	}
	\caption{Marginal cost $\mathcal{C}(X \cup v_{i,j})-\mathcal{C}(X)$  (in red)  when the vertex $v=v_{i,j}$ is added and $i_{\max}=4$ \revision{with unitary costs}:
		 (a) $i > i_{\max}$ \revision{with marginal cost $8$};
		 (b) $i < i_{\max}$ \revision{with marginal cost $6$}, no vertex of row $i$ belongs to $S$;
		 (c) $i \le i_{\max}$ \revision{with marginal cost $2$}.}
	\label{fig:cases}
	
\end{figure}

Consider for instance Figure~\ref{fig:TX}. 
The cycle $\mathcal{C}(X)$ traverses, in the order, row $1$ back-and-forth up to vertex $v_{1,4}$, row $2$ back-and-forth up to vertex $v_{2,3}$, vertex $v_{3,1}$ in row $3$, and row $4$ back-and-forth up to vertex $v_{4,3}$.
Finally, it returns to $v_{1,1}$ traversing back column $1$.
The cost of $X$ in Figure~\ref{fig:TX} is $2(|T_X|-1)=20$ (precisely, $3 \cdot 2 + 1 + 2 \cdot 2 + 2 + 2 \cdot 2 + 3$).
Since $T_X$ is uniquely defined given $X$, for convenience, we define, in Eq.~\eqref{eq:submodular-orienteering-subset}, $\mathcal{C}(X)$ as the cost of visiting the tree $T_X$, and the function $\mathcal{R}(X)$ as the sum of the rewards of all the vertices $v \in X$.
Note that $\mathcal{R}(T_X) \ge \mathcal{R}(X)$ because $X \subseteq T_X$.

It is very important to note that the marginal cost of $\mathcal{R}$ and $\mathcal{C}$, i.e., the cost for updating both $\mathcal{R}$ and $\mathcal{C}$ after the set $X$ is increased by one vertex, can be computed in constant time.
This is trivial for $\mathcal{R}$, indeed $\mathcal{R}(X \cup v)=\mathcal{R}(X)+\mathcal{R}(v)$.
The increment of the cost of the $\mathcal{C}(X)$ when a new vertex is added depends on the position of $v$ with respect to $X$. 
The possible scenarios are illustrated, for example, in Figure~\ref{fig:cases} and the increment is depicted in red.
Figure~\ref{fig:cases} illustrates how the cost changes depending on the relative position of $v=v_{i,j}$ with respect to the last line in the solution set $S$.
It is easy to compute the increment in constant time.
Specifically, in Figure~\ref{fig:tree-a} (case $i > i_{\max}$), the increment of cost considers the vertical movements for reaching the new $i$-th row and the horizontal movements for visiting vertices up to $v$;
in Figure~\ref{fig:tree-b} (case $i < i_{\max}$ with no vertex of row $i$ belongs to $S$), the increment of cost only considers the full horizontal movements for visiting vertices up to $v$;
in Figure~\ref{fig:tree-c} (case $i \le i_{\max}$), the increment of cost only considers the partial horizontal movements for completing the visit of the $i$-th row up to $v$.
From the above considerations, Eq.~\eqref{eq:sumbodularity def 2} is satisfied as equality from both  $\mathcal{R}$ and $\mathcal{C}$, and thus such functions are modular.


\section{Optimal Solution for \prob}\label{sec:opt-alg}
We devise a dynamic programming algorithm, called \optsclong (\optsc), that optimally solves \prob in polynomial time.

During the initialization, two tables $T$ and $R$ of size $m \times n$ and $m \times (\frac{B}{2}+1)$, respectively, are created.
Table $R$ has columns $j=0, 1, \ldots, \frac{B}{2}$.
Instead, Table $T$ is initialized as follows: $T[i,j] = \sum_{k=1}^{j} \mathcal{R}(v_{i,k})$ for each $i,j$ which represents, fixed a row $i$, the cumulative reward up to the $j$-th column starting from the leftmost side, with $1 \le j \le n$. Clearly, $T[i,1]=0$.
\revision{The initialization of table $T$ has time and space cost $\mathcal{O}(mn)$ since its size is $m \times n$ and the cells of each row can be filled in constant time from left to write:
$T[i,j] = T[i,j-1]+ \mathcal{R}(v_{i,j})$.}
The entry $R[i,b]$ records the largest reward that can be attained with budget $2b$, where $0\leq b\leq \frac B 2$, visiting only the first $i$ rows.
Finally, let $Q[i,b]$ be the last column visited in row $i$ to obtain the reward $R[i,b]$.

Some general properties will be exploited by our dynamic programming solution:
\begin{itemize}
    \item If $B\geq 2nm+2(m-1)$ then the robot can visit the whole graph.
    \item Any optimal solution  traverses the selected rows either in increasing or decreasing order with respect to their row indices. 
    \item Any optimal solution visits each row at most once. 
    \item For each row, the algorithm has to select the last vertex to visit. 
    If no vertex is visited, still $v_{i,1}$ has to be traversed to reach the subsequent row $i+1$.
\end{itemize}

The optimal solution that visits the first $j$ vertices in row $i$, with $j \ge 1$, gains $T[i,j]$ reward from row $i$, and spends $2j$ budget to serve rows $i$ starting from row $i-1$.
Namely, we reserve $2$ units of budget to change row and $2(j-1)$ units to traverse row $i$.
We thus have the next recurrence.
The first row of table $R$ is defined as follows:
\begin{equation*}
    R[1,b] = 
        \begin{cases}
             T[1, b+1] & 0 \le b \le n-1\\
             T[1, n] & n \le b \le \frac{B}{2}
        \end{cases}
\end{equation*}
Note that for $b=0$, $T[1,0]=\mathcal{R}(v_{1,1})=0$.
Then, for each subsequent row $R[i,b]$, $i \ge 2$:
\begin{equation}\label{eq:dp-recurrence}
    R[i,\!b]\!=\!\!
        \begin{cases}
            -\infty &  \!\!\!\!b\!<\!i\!-\!1 \\
            0 & \!\!\!\!b\!=\!i\!-\!1\\
                            \max\limits_{1 \le j \le \min\{b-i+2, n-1\}} \!\!\{\!R[i\!-\!1, b\!-\!j\!]\!+\!T[i, j]\!\} & \!\!\!\!b\!>\!i\!-\!1 
        \end{cases}
\end{equation}

Note that $j \le \min\{b - i+1, n-1\}$ is obtained by observing that $j \le n-1$, $b - (j - 1)-1 \ge 0$ and more strictly $b-(j-1)-1 \ge i-2$ to have enough budget to reach row $i-1$. 
If $b=i-1$ and $i \ge 2$, $R[i,b]$ is  $0$ because there is only enough budget to reach row $i$ traversing column $1$.

Table $Q$ is then filled recalling for each position the column that has given the maximum reward, i.e., $Q[i,b] \gets j = \argmax R[i,b]$.
Finally, the reward of the optimal solution with budget $\frac{B}{2}$ is found by calculating $\max_{1 \le i \le m} R[i, \frac{B}{2}]$ because the furthest row reached by the optimal solution, is not known in advance.
Having fixed the last row, the solution is computed in $\mathcal{O}(m)$ tracing back the choices using the table $Q$.
Note that with a single execution of \optsc, the optimal reward for any value $b$, with $0 \le b \le \frac{B}{2}$, is computed. 
That is, the maximum reward for budget $b$ is the maximum in column $b$ of table $R$.

The algorithm runs in time $\mathcal{O}(m  n  \frac{B}{2})$ plus $\mathcal{O}(m)$ to retrieve the solution, and takes $\mathcal{O}(m  n + m  \frac{B}{2})$ space.
Since the maximum budget is $B\geq 2nm+2(m-1)$, i.e., $B$ is upper bounded by $\mathcal{O}(mn)$, then algorithm \optsc is strictly polynomial in the size of the input.

By the above discussion, the next theorem can be stated.

\begin{theorem}\label{th:bound}
Algorithm \optsc optimally solves \prob in time $\mathcal{O}(m  n  \frac{B}{2})$.
\end{theorem}

\begin{proof}
The running time is obvious. 
Let us define $ R(i,b) $ as the optimal profit to reach any vertex of row $i$ when the budget is $2 b$.
The solution $Q(i,b)$ associated with $R(i,b)$ has maximum profit and must traverse the vertices $v_{1,1}, \ldots, v_{i,1}$ of $c_1$, i.e., the first column with no rewards.

For the correctness, we use induction.
For the base case, $i=1$ and any $b$, or any row $i$ and $b \le i - 1$, the correctness follows from the above discussion.

Inductive Step: When computing $ R[i,b] $ by the induction hypothesis, we have that $R[i-1, b-j] $ for any $1 \le j \le \min\{ b-i+2, n-1 \}$ are already computed correctly.
Since any optimal solution $Q[i,b]$ visits the rows in increasing index order, traverses only once row $i$, up to any $c_j$ (recall that $j=1$ means that the row is not visited), and must traverse the vertices $v_{1,1}, \ldots, v_{i,1}$ of $c_1$, $Q[i,b]$ is built starting from a sub-problem that traverses the vertices $v_{1,1}, \ldots, v_{i-1,1}$ and some vertices of row $i$. Then, $Q[i,b]$ is based on a sub-problem that considers up to row $i-1$, leaves $1$ unit of budget to reach row $i$, and leaves $(j-1)$ units of budget to reach vertex $j$ in row $i$, for some $1 \le j \le b-i$.
Hence, the value of $R[i,b]$ in Eq.~\eqref{eq:dp-recurrence} is correct.

Now, assume by contradiction, that there exists a solution $Q'[i,b] \not =  Q[i,b] $  with cost $R'[i, b] > R[i, b]$, and that is the first time that Eq.~\eqref{eq:dp-recurrence} does not provide the optimum.
Let vertex $v_{i,j}$ be the vertex reached by $Q'[i,b]$ in row $i$. 
The solution $Q'[i,b]$ gains $T[i,j]$, and $R'[i,b]-T[i,j]$ corresponds to the profit of a sub problem $Q[i-1,b-j]$ for which $R[i-1, b-j]$ is optimal.
That is, $R'[i,b]-T[i,j]=R[i-1, b-j]$ and thus $R[i,b]=R[i-1, b-j]+T[i,j]=R'[i,b]$.

Finally, since we do not know in advance the furthest row that belongs to the optimal solution, the optimal solution with budget $B$ is found in column $\frac B 2$, and precisely it is $\max_{1 \le i \le m} R[i, \frac B 2]$.
\end{proof}

\paragraph{Example}
\revision{For the example in Figure~\ref{fig:graph-sc}, assuming a budget $B=16$, the Table $R$ is iteratively created from its first row.

\begin{table}[htbp]
	\renewcommand{\arraystretch}{1.25}
	\centering
	\begin{tabular}{l|ccccccccc}
        $R$ & 0 & 1 & 2 & 3 & 4 & 5 & 6 & 7 & 8 \\
        \hline
        1 & 0 & 3 & 4 & 8 & 9 & 9 & 9 & 9 & 9 \\
	\end{tabular}
\end{table}
Here, we assume to visit only the first row.
Specifically, $R[1,0]=0$ because with no budget there is no reward; $R[1,1]=3$ because with a budget of $2\cdot1 = 2$, the maximum obtainable reward is $3$; $R[1,2]=4$ because with a larger budget $2\cdot2 = 4$, it is possible to visit the third column, and hence the cumulative reward is $3+1=4$; $R[1,3]=8$ because spending a budget of $2\cdot3 = 6$, the cumulative reward is $3+1+4=8$, and so on.
Clearly, $R[1,4] = \ldots = R[1,8] = 9$ because the available budget is sufficiently enough for completely visiting the first row as a whole.
The second row of table $R$ is the following one.

\begin{table}[htbp]
	\renewcommand{\arraystretch}{1.25}
	\centering
	\begin{tabular}{l|ccccccccc}
        $R$ & 0 & 1 & 2 & 3 & 4 & 5 & 6 & 7 & 8 \\
        \hline
        2 & $-\infty$ & 0 & 3 & 4 & 11 & 17 & 20 & 21 & 25 \\
	\end{tabular}
\end{table}
Assuming to visit also the second row, $R[2,0]=-\infty$ because, with no budget, there is no chance for computing a cycle that visits any vertex in the second row and going back to $v_{1,1}$.
$R[2,1]=0$ because with a budget of $2$, it is possible only to visit $v_{2,1}$ and go back, without any reward since the elements on the first column do not provide reward.
The first interesting value is $R[2,4]=11$, because with budget $8$, forcing to visit the second row, it is possible to arrive up to $v_{2,4}$ which gives a cumulative reward of $1+1+9=11$, which is the maximum possible.
Then, $R[2,7]=21$ because with budget $14$, visiting $v_{1,2}$ with cumulative reward of $3+1=4$, and visiting $v_{2,5}$ with cumulative reward $1+1+9+6=17$, the total reward is $4+17=21$.

The same strategy, applying Eq.~\eqref{eq:dp-recurrence}, permits to build the final table $R$.
\begin{table}[htbp]
	\renewcommand{\arraystretch}{1.25}
	\centering
	\begin{tabular}{l|ccccccccc}
        $R$ & 0 & 1 & 2 & 3 & 4 & 5 & 6 & 7 & 8 \\
        \hline
        1 & 0 & 3 & 4 & 8 & 9 & 9 & 9 & 9 & 9 \\
        2 & $-\infty$ & 0 & 3 & 4 & 11 & 17 & 20 & 21 & 25 \\
        3 & $-\infty$ & $-\infty$ & 0 & 3 & 10 & 19 & 28 & 31 & 32 \\
        4 & $-\infty$ & $-\infty$ & $-\infty$ & 0 & 3 & 10 & 19 & 28 & 31
	\end{tabular}
\end{table}

The maximum achievable reward assuming to visit only the first row is $9$, i.e., full visit of the row because there is available budget.
In the case of considering also the second row, it is not possible to do two full visits for both the first and second rows, since the cost will be larger than the budget, hence the maximum achievable reward is $25$.
The same could be done up to the third row, i.e., maximum budget $32$, and finally up to the fourth, i.e., $31$.
It is easy to see that the optimal solution visits only the first three rows obtaining a total reward of $32$.

The optimal solution is extracted from table $Q$, which is:
\begin{table}[htbp]
	\renewcommand{\arraystretch}{1.25}
	\centering
	\begin{tabular}{l|ccccccccc}
        $Q$ & 0 & 1 & 2 & 3 & 4 & 5 & 6 & 7 & 8 \\
        \hline
        1 & 1 & 2 & 3 & 4 & 5 & 5 & 5 & 5 & 5 \\
        2 & $-\infty$ & 1 & 1 & 1 & 4 & 5 & 5 & 5 & 5 \\
        3 & $-\infty$ & $-\infty$ & 1 & 1 & 3 & 4 & 5 & 5 & 5 \\
        4 & $-\infty$ & $-\infty$ & $-\infty$ & 1 & 1 & 1 & 1 & 1 & 1
	\end{tabular}
\end{table}

The solution is extracted from the index in $Q$ that has the maximum reward in $R$, i.e., $Q[3,8] = 5$, so $v_{3, 5}$ is visited.
Then, the farthest column to be visited from the previous row is $Q[3-1,8-Q[3,8]] = Q[2,3] = 1$, i.e., up to $v_{2,1}$.
Finally, the last column of row $1$ to be visited is $Q[2-1,3-Q[2,3]] = Q[1, 2]$, i.e., up to $v_{1,2}$.
In conclusion, the optimal solution is: $S = \{  v_{1,1},v_{1,2}, v_{1,3}, v_{2,1}, v_{3,1}, v_{3, 2}, v_{3, 3}, v_{3, 4}, v_{3, 5}\}$.
}


\section{Proposed Greedy Algorithms}\label{sec:algorithms}

In this section, we propose two basic greedy algorithms, called \gdymelong (\gdyme) and \gdymclong (\gdymc), for \prob; 
and two additional greedy approximation algorithms, called \apxmrelong (\apxmre), and \apxmrclong (\apxmrc).
We also compare the time and space complexity of these algorithms with the optimal algorithm, \optsclong (\optsc).

\subsection{The \gdyme Algorithm}
The \gdymelong (\gdyme) algorithm is a simple greedy algorithm for \prob.
The main idea behind it is to visit the vertices in decreasing order of their rewards, i.e., vertex $v_{i,j}$ with current largest reward is either added to the solution set $S$, if it is reachable with the residual budget or discarded.
In addition to the vertex $v_{i,j}$ selected by the greedy strategy, at each step, the robot collects on row $i$ all the available rewards from $v_{i,1}, \ldots, v_{i,j}$. 
This is a simplification of the fact that at some point all vertices passed through to reach $v_{i,j}$ would be chosen by \gdyme as they require no budget to be reached.
Hence, selected the vertex $v_{i,j}$, the current solution set $S$ adds $v_{i,j}$ along with all the vertices $v_{i,k}$ with $1 \le k < j$ of the $i$-th row.
Recall that the reward of a vertex is collected only once even if the robot passes through it several times.

\begin{algorithm}[ht]
	$S \gets v_{1,1}$\;\label{code:1-set-params}
	$R \gets \textsc{buildMHeap}(\mathcal{R}(v_{i,j})), 1 \le i \le m, 1 \le j \le n$\;\label{code:1-sort}
	\While{$ \mathcal{C}(S) \le B $} {\label{code:1-loop}
		$v_{i,j} \gets \textsc{extractArgMax}(R)$\;\label{code:1-max}
		\If{$ \mathcal{C}(S \cup v_{i,j}) \le B$} {\label{code:1-check}
			$S \gets S \cup v_{i,k}, {1 \le k \le j}$\;\label{code:1-updates}
			$\textsc{extractHeap}(v_{i,k}), 1 \le k \le j$\;
		}
	}
	\Return{$\mathcal{R}(S)$}\label{code:1-return}
	\caption{\gdyme\(( \text{Graph } A, \text{Budget } B)\)}
	\label{alg:gdyme}
\end{algorithm}

Algorithm~\ref{alg:gdyme} first initializes $S$ with the home vertex $v_{1,1}$, whose reward is null by definition (Line~\ref{code:1-set-params}) and subsequently creates a max-heap $R$ with rewards of all the vertices in $V$ (Line~\ref{code:1-sort}).
This preprocessing takes $\mathcal{O}(m n)$ time.

Next, the main loop (Line~\ref{code:1-loop}) iteratively selects and extracts from \revision{the heap structure $R$} the root $v_{i,j} \in V$ (Line~\ref{code:1-max}), and verifies whether $v_{i,j}$ is reachable with respect to the available budget or not (Line~\ref{code:1-check}).
Recalling that the minimum path to visit a set of vertices of the aisle-graph $A$ is unique, 
the cost $\mathcal{C}(S \cup v_{i,j})$
can be computed in constant time incrementally from $\mathcal{C}(S)$ (see Figure~\ref{fig:cases}).

As already pointed out, if adding vertex $v_{i,j}$ is feasible (from Line~\ref{code:1-updates}), 
all the vertices $v_{i,k}, 1 \le k \le j$, will be added to $S$.
Namely, adding these vertices does not require extra budget because they are on the path to $v_{i,j}$. 
At the end of the main loop, 
the total reward $\mathcal{R}(S)$ is returned (Line~\ref{code:1-return}) in $\mathcal{O}(|S|)$ time.

The main loop (Line~\ref{code:1-loop}) is executed at most $\frac{B}{2}$ times
because the addition of a new vertex to $S$ requires at least $2$ units of budget.
Since each item inserted in $S$ is extracted from $R$,
the algorithm spends $\mathcal{O}(|S|\log(mn))$ time to handle $R$.
Clearly $|S|$ is upper bounded by $B$. Moreover $B$ can be obviously considered as upper bounded 
by $\mathcal{O}(mn)$, since the minimum budget that allows to visit the whole input aisle-graph is $2(mn+m)$.
Any larger budget would lead to have a trivial solution, i.e., a full visit. 
Hence, the total time complexity of \gdyme is 
$\mathcal{O}(m n + B + |S|\log(mn)) = \mathcal{O}(mn\log(mn))$; and its space complexity is $\mathcal{O}(mn)$.

\paragraph{Example}
For the example in Figure~\ref{fig:graph-sc} and a budget $B=16$, the first element greedly added to the solution set $S$ is $v_{2,4}$ with value $9$.
Consequently, all the other elements from the $2$-th row are put in $S$.
From $v_{2,4}$ there is enough budget for visiting other vertices, and in fact the next available vertex the with largest reward is $v_{3,4}$, with reward $9$.
Now, selecting $v_{3,4}$, there is no sufficient budget for visiting other vertices.
Finally, the solution is $S=\{ v_{1,1}, v_{2,1}, v_{2,2}, v_{2,3}, v_{2,4}, v_{3,1}, v_{3,2}, v_{3,3}, v_{3,4} \}$ with total reward $30$.

\subsection{The \gdymc Algorithm}
The \gdymclong (\gdymc) algorithm keeps the same greedy strategy as \gdyme, but selects the vertex with the maximum cumulative reward.
In the preprocessing step,  for each vertex $v_{i,j}$, the algorithm initially sums up in a matrix $R$ the cumulative reward $R[i,j]=\sum_{k=1}^{j}\mathcal{R}(v_{i,j})$, i.e., the sum of the rewards of all the vertices belonging to the $i$-th row up to the $j$-th column.
There is only one  vertex with maximum cumulative reward to be selected in each row, and at the beginning, this vertex is $v_{i,n}$ because the cumulative reward is a monotone function.
During the execution of the algorithm, the candidate vertices, one per each row, are stored in a max-heap.

\begin{algorithm}[ht]
	$S \gets v_{1,1}$\;\label{code:2-set-params}
	$R[i,j] \gets \sum_{k=1}^{j}{\mathcal{R}(v_{i,k})}, 1 \le i \le m, 1 \le j \le n $\;\label{code:2-cumulative}
	$M \gets \textsc{buildHeap}(R[i,n]), 1 \le i \le m$\;\label{code:2-heap}
	\While{$ \mathcal{C}(S) \le B $} {\label{code:2-loop}
		$v_{i,j} \gets \textsc{extractArgMax}(M)$\;\label{code:2-max}
		\If{$ \mathcal{C}(S v_{1,1}\cup v_{i,j}) \le B$} {\label{code:2-check}
		    $S \gets S \cup_{1 \le k \le j} v_{i,k}$\label{code:2-updates}
		} \Else{\label{code:2-else}
			$\textsc{insertHeap}(M, R[i,j-1])$\label{code:2-add}
		}
	}
	\Return{$\mathcal{R}(S)$}\label{code:2-return}
	\caption{\gdymc $(\text{Graph } $A$, \text{Budget } B)$}
	\label{alg:gdymc}
\end{algorithm}

Algorithm~\ref{alg:gdymc} sets the initial solution $S=\{v_{1,1}\}$ (Line~\ref{code:2-set-params}), and computes the cumulative rewards row by row (Line~\ref{code:2-cumulative}).
Then, a max-heap $M$ is built (Line~\ref{code:2-heap}) using  the 
largest cumulative rewards of each row, i.e., $R[i,n]$, for $1 \le i \le m$.
This preprocessing phase takes $\mathcal{O}(m n)$ in time and space.

The main greedy loop iterates until the solution exhausts the budget for the robot (Line~\ref{code:2-loop}). 
At each step, the vertex $v_{i,j} \in V$ with the largest value $R[i,j]$ in $M$ (Line~\ref{code:2-max}) is selected.
If there is enough budget to visit such a vertex, $v_{i,j}$ is added to the current solution $S$ along with all the vertices on its left side (Line~\ref{code:2-updates}) on row $i$.
Note that no other vertices in row $i$ can be inserted later and the size of $M$ is decreased by one. If the budget condition (Line~\ref{code:2-check}) is not satisfied, it implies that there is not enough budget to include $v_{i,j}$ in $S$, and hence $v_{i,j}$ is discarded.
In this case, notice that not even the elements on the right of $j$ can  be reached.
Thus, the remaining largest cumulative reward $R[i,j-1]$ in row $i$, which is associated with vertex $v_{i,j-1}$ on the left of the vertex just discarded, is inserted into $M$ (Line~\ref{code:2-add}).

When the budget is exhausted, the total reward $\mathcal{R}(S)$ is returned (Line~\ref{code:2-return}).

The main loop (Line~\ref{code:1-loop}) is executed at most $\frac{B}{2}$ times.
The cost for extracting the maximum element of a heap or inserting a new element into a heap of size $m$ is $\mathcal{O}(\log(m))$, and the time for updating $S$ and computing $\mathcal{R}(S)$ is bounded by the size of $S$ which is $\mathcal{O}(B)$. Thus, the total time complexity is $\mathcal{O}(m n + B\log(m))$ and space complexity is $\mathcal{O}(mn+m)$.

\paragraph{Example}
For the example in Figure~\ref{fig:graph-sc} and a budget $B=16$, the vertex with the largest cumulative reward within the available budget is $v_{3, 5}$ with an overall cumulative $28$.
Then, the next vertex is $v_{1, 3}$ with overall value of $4$.
After that, no any other vertex can be chosen and the total obtained reward is $32$ and $S = \{ v_{1,1}, v_{1, 2}, v_{1, 3}, v_{3,1}, v_{3, 2}, v_{3, 3}, v_{3, 4}, v_{3, 5} \}$. In this case, \gdymc returns the optimal reward.

\subsection{The \apxmre Algorithm}
The \apxmrelong (\apxmre) algorithm selects, at each step, the feasible vertex that maximizes the ratio between its reward and the cost (budget) to add the selected vertex to the current greedy solution. 
In addition to the greedy solution, \apxmre also builds a second solution that contains the item with the maximum reward reachable with the budget $B$. 
The algorithm returns the maximum between the two solutions.
We will prove in Section~\ref{sec:approximation} that \apxmre guarantees an approximation ratio of $\frac{1}{2}(1-\frac{1}{e})$.

\begin{algorithm}[ht]
	$S_1 \gets \emptyset, r_{\max} \gets -\infty, S_2 \gets v_{1,1}$\;\label{code:3-set-params}
	$H \gets [i,d][j], 1 \le i \le m, 0 \le d \le i-1, 1 \le j \le n$\;\label{code:3-matrix}
	\For{$ i \gets 1, m $} {\label{code:3-preprocessing}
		\For{$ d \gets 0, i-1 $} {
			\For{$ j \gets 1, n $} {
				$H[i, d][j] \gets \frac{\mathcal{R}(v_{i,j})}{2(j-1) + 2d}$\;\label{code:3-ratio}
				\If{$2(j-1) + 2d \le B$}{\label{code:3-sol-alone}
					\If{$\mathcal{R}(v_{i,j}) > r_{\max}$}{
						$r_{\max} \gets \mathcal{R}(v_{i,j})$\;
				 		$S_1  \gets v_{i,j}$\;
					}
				}
			}
		}
	}
	$ i_{\max} \gets 1$\;\label{code:3-max-row}
	\While{$ \mathcal{C}(S_2) \le B $} {\label{code:3-loop}
	    $M \gets [i], 1 \le i \le m$\;\label{code:3-maximums}
		\For{$ i \gets 1, m $} {\label{code:3-loop-m}
			\If{$ i \le i_{\max} $} {\label{code:3-loop-maximums}
				$M[i] \gets \textsc{max}(H[i, 0])$\;\label{code:3-zero}
			} \Else {
				$M[i] \gets \textsc{max}(H[i, i - i_{\max}])$\;\label{code:3-downshift}
			}
		}
		$v_{i,j} \gets \textsc{argMax}(M)$\;\label{code:3-selected}
		\If{$ \mathcal{C}(S_2 \cup v_{i,j}) \le B$} {\label{code:3-check}
			$S_2 \gets S_2 \cup_{1 \le k \le j} v_{i,k}$\;\label{code:3-updates}
			\For{$ k \gets j+1, n $} {
				$H[i, 0][k] \gets \frac{\mathcal{R}(v_{i,j})}{2(j-k)}$\;\label{code:3-update-heap}
			}
		}
		$ i_{\max} \gets \textsc{max}(i, i_{\max})$\;\label{code:3-new-imax}
	}
	\Return{$\max\{\mathcal{R}(S_1), \mathcal{R}(S_2)\}$}\label{code:3-return}
	\caption{\apxmre$( \text{Graph } $A$, \text{Budget } B)$}
	\label{alg:apxmre}
\end{algorithm}

Algorithm~\ref{alg:apxmre} initializes two solutions $S_1$ and $S_2$, and $r_{\max}$ which stores the maximum current reward (Line~\ref{code:3-set-params}). 
Here $S_1$ is built in the preprocessing phase by checking if there is enough budget to reach each vertex from the home $v_{1,1}$ and go back.
To compute $S_2$, a two-dimensional matrix $H$ of pointers is constructed (Line~\ref{code:3-matrix}).
Each pointer $H[i,d]$ points to a vector of length $n$.
In the preprocessing phase (Line~\ref{code:3-preprocessing}), 
the vectors pointed by $H[i,d]$, for $1 \le i \le m$ and $0 \le d \le i$, are initialized. In Line~\ref{code:3-ratio}, the  vector of $n$ positions pointed by $H[i,d]$ stores in position $j$, for $1 \le j \le n$, the ratio between the reward $\mathcal{R}(v_{i,j})$ and the length $2(j-1)+2d$ of the cycle from vertex $v_{i-d,1}$ to $v_{i,j}$.
Such a cycle consists of a vertical subpath of length $d$ and a horizontal subpath of length $j-1$. 

In other words, the vector pointed by $H[i,d]$ stores the cost of reaching the vertices in row $i$ when the farthest row already visited is row $i_{\max} = i-d$. 
In this case, to reach a vertex of row $i$ we have to consider extra  $d$ vertical steps.
While $S_2$ grows, the farthest row in $S_2$ increases and the cost for reaching the vertices of row $i$  changes (see Figure~\ref{fig:cases}).
The ascending order of the ratios (and in particular the maximum) is
however not preserved when the traveling cost (i.e., the ratio denominator) changes due to the change of the last row \imax of the solution. 
Therefore, for each row $i$, $1 \le i \le m$, we precompute the ratios considering all the possible distances $H[i,d]$, $0 \le i \le d$.
This preprocessing phase takes $\mathcal{O}(m^2 n)$ in time and space.

After the preprocessing procedure (loop Line~\ref{code:3-preprocessing}), the current furthest visited row \imax (Line~\ref{code:3-max-row}) is set to $1$ and the main loop starts (Line~\ref{code:3-loop}).
A vector $M$ of length $m$ is created (Line~\ref{code:3-maximums}) using for each row the root of $H[i,i-1]$.
Then vertex $v_{i,j}$ in $M$ with the largest ratio is greedily chosen (Line~\ref{code:3-selected}) and inserted in $S_2$ if feasible.
The current solution set $S_2$ is updated  including the vertices in row $i$ on the left of $v_{i,j}$  (Line~\ref{code:3-updates}).
Moreover, for any vertex in row $i$ on the right of the selected vertex, i.e., $v_{i,k}$ for $j+1 \le k \le n$, the ratios are updated in $H[i,0]$  (Line~\ref{code:3-update-heap}). 
This is because $S_2$ already reached the vertex $v_{i,j}$ in row $i$, and only a small horizontal distance has to be traversed to reach any remaining vertex in row $i$.
Finally, the value of the maximum visited row \imax is updated (Line~\ref{code:3-new-imax}).

Before proceeding to the next vertex selection, for each vertex $v_{i,j}$ we reason about the budget to be spent for the vertical movements. 
Indeed, if $i \le i_{\max}$ (Line~\ref{code:3-zero}), given the  solution $S_2$ already reaches $i$, the actual ratios for row $i$ are pointed by $H[i,0]$ which considers only the horizontal distances to reach vertices $v_{i,j}$  (Line~\ref{code:3-downshift}). 
If $i > i_{\max}$, some budget for the vertical steps has to be spent since row $i$ has never been reached.
Hence, the actual ratios for row $i$ are pointed by $H[i,i - i_{\max}]$ which considers  $i - i_{\max}$ vertical distance.
In practice, since the last row in $S_2$ varies, we change the costs of reaching the vertices, as explained in Figure~\ref{fig:cases}.

At the very end, the algorithm returns the best solution between subsets $S_1$ and $S_2$ (Line~\ref{code:3-return}).

The main cycle (Line~\ref{code:3-loop}) is executed at most $\frac{B}{2}$ times, resulting in $\mathcal{O}(B(m+n))$ time because the cycle (Line~\ref{code:3-loop-m}) costs only $\mathcal{O}(m)$ time.
Thus the overall time and space complexity are $\mathcal{O}(m^2n+B(m+n))$ and $\mathcal{O}(m^2n)$, respectively.

\paragraph{Example}
For the example in Figure~\ref{fig:graph-sc} and a budget $B=16$, the vertex with the largest ratio between its reward and its cost for adding it at the solution set $S$ is $v_{1, 2}$ with ratio $1.5$ (i.e., reward $3$ and cost $2$).
The next vertex selected is $v_{2, 4}$ with ratio $9/8 = 1.125$.
The cost is $8$ because in the current solution there are only vertices belonging to the first row, so the costs for exploring the second row take into account two more vertical units.
Now, the new vertex to add is $v_{2, 5}$ with ratio $6/2 = 3$.
This is due to the fact that $6$ is its reward, but only $2$ is the additional cost for visiting it, since $v_{2, 4}$ already belongs to $S$.
The next possible vertex is $v_{1, 4}$, and hence the solution has total reward $25$ and $S = \{ v_{1,1}, v_{1, 2}, v_{1, 3}, v_{1, 4}, v_{2, 1}, v_{2, 2}, v_{2, 3}, v_{2, 4}, v_{2, 5}\}$.


\subsection{The \apxmrc Algorithm}
The \apxmrclong (\apxmrc) algorithm is a greedy strategy that runs similar to \apxmre algorithm.
The only difference in the greedy rule is that the subvectors (Algorithm~\ref{alg:apxmre}, Line~\ref{code:3-ratio}) consider, at the numerator, the cumulative reward instead of the single vertex reward, i.e.,
\[
    H[i, d][j] \gets \frac{\sum_{k=1}^{j}{\mathcal{R}(v_{i,k})}}{2(j-1) + 2d}.
\]
Accordingly, \apxmrc has the same time and space complexity as \apxmre.

\paragraph{Example}
For the example in Figure~\ref{fig:graph-sc} and a budget $B=16$, the selected vertex with the maximum ratio between the cumulative reward and the cost for reaching it is $v_{3, 5}$ with ratio $28/12 = 2.333$.
Then, the next one is $v_{1, 2}$ with ratio $3/2 = 1.5$.
Finally, the next one is $v_{1, 3}$ with ratio $1/2 = 0.5$.
This is due to the fact the solution already had included the vertex $v_{1, 2}$. 
In conclusion, the solution has total reward $32$ and $S = \{ v_{1,1}, v_{1, 2}, v_{1, 3}, v_{3,1}, v_{3, 2}, v_{3, 3}, v_{3, 4}, v_{3, 5}\}$. 
In this case, \apxmre returns the optimal reward.

\subsection*{Comparison of Algorithms}
In Table~\ref{tab:comparison_algorithms} we summarize the time and space complexity of all the algorithms we presented, distinguishing between the time complexity of the preprocessing phase and the main algorithm phase.
Recalling that $B$ is upper bounded by $\mathcal{O}(mn)$, when $B=\Theta(mn)$, \optsc requires $\mathcal{O}(\max\{m,n\}^4)$. 
Thus, all our non-optimal algorithms are faster than \optsc, even including the preprocessing time. 
The most expensive greedy algorithms are \apxmre and \apxmrc requiring $\mathcal{O}(\max\{m,n\}^3)$ operations.
The time complexity of \apxmre and \apxmrc is dominated by the preprocessing time which is larger than that of \optsc.
Algorithms \gdyme and \gdymc are always faster than \optsc, requiring $\mathcal{O}(\max\{m,n\}^2\log m)$ time, but they slightly lose in reward-performance.
In terms of space, \gdyme and \gdymc are more efficient than \optsc which requires the same space as \apxmre and \apxmrc.

\begin{table}[htbp]
\caption{Comparison between the algorithms.}
	\label{tab:comparison_algorithms}
	\centering
	\begin{tabular}{c|ccc}
		\hline
		\multirow{2}{*}{Algorithm} & \multicolumn{2}{c}{Time complexity} & \multirow{2}{*}{Space complexity} \\
		& Preprocessing & Main & \\
		\hline
		\optsc & $\mathcal{O}(mn)$  & $\mathcal{O}(Bmn)$ & $\mathcal{O}(mn+Bm)$ \\
		\gdyme & $\mathcal{O}(mn)$  & $\mathcal{O}(B\log(mn))$ & $\mathcal{O}(mn)$ \\
		\gdymc & $\mathcal{O}(mn)$  & $\mathcal{O}(B\log(m))$ & $\mathcal{O}(mn+m)$ \\
		\apxmre & $\mathcal{O}(m^2n)$  &  $\mathcal{O}(B(m+n))$ & $\mathcal{O}(m^2n)$ \\
		\apxmrc & $\mathcal{O}(m^2n)$  &  $\mathcal{O}(B(m+n))$ & $\mathcal{O}(m^2n)$ \\
		\hline
	\end{tabular}
\end{table}

\section{Guaranteed Bound for Approximation Ratio} \label{sec:approximation}
%
\revision{In this section, we analyze the performance guarantee of our approximation algorithms.
Specifically, we initially highlight how the selected element by the greedy step improves the objective function in Lemma~\ref{lemma:marginal gain}, then we show  how the partial greedy solution is bounded with respect to the optimal solution in Lemma~\ref{lemma:distance from optimum}, and finally we prove the guaranteed approximation bound of algorithms \apxmre and \apxmrc exploiting the two aforementioned Lemmas in Theorem~\ref{th:bound}.}


Let \(G_i\) be the solution obtained from \apxmre at step \(i\).
Moreover, at each iteration \(i\), \apxmre adds to the solution set \(G_{i-1}\) the element \({x_i}\) that maximizes the ratio between the marginal gain and the marginal cost of adding such an element (see Algorithm~\ref{alg:apxmre}, Line~\ref{code:3-selected}). 
That is,
\begin{equation}
    \label{eq:passo greedy}
    x_i = \argmax_{v_{i,j} \in V} \frac{\mathcal{R}(G_{i-1} \cup v_{i,j}) - \mathcal{R}(G_{i-1})}{\mathcal{C}(G_{i-1} \cup v_{i,j}) - \mathcal{C}(G_{i-1})}
\end{equation}
For our problem \(\mathcal{R}(G_{i-1} \cup v_{i,j}) - \mathcal{R}(G_{i-1}) = \mathcal{R}(x_i)\), due to the modularity property. 
In Lemma~\ref{lemma:marginal gain} we prove that the inclusion of the element selected by the greedy step (Eq.~\eqref{eq:passo greedy}) improves the objective function by at least a quantity proportional to the current distance to the optimum. 
Namely, we are able to bound the improvement of the greedy step from below the improvement on the objective function.
Let the solution \(G_{\ell+1}\) from \apxmre the first iteration when it violates the budget constraint and stops. 
\begin{lemma}\label{lemma:marginal gain}
For any \(i = 1, \ldots, \ell+1\), it holds that
\[
\mathcal{R}(G_i) - \mathcal{R}(G_{i-1}) \ge \frac{\mathcal{R}(\opt) - \mathcal{R}(G_{i-1})}{B} (\mathcal{C}(G_i) - \mathcal{C}(G_{i-1})),
\]
where \(\opt\) is the optimal solution considering budget \(B\).
\end{lemma}
\begin{proof}
Suppose \(\opt\) is the optimal solution considering budget \(B\).
For any \(i = 1, \ldots, \ell+1\) we have
\begin{align*}
    \mathcal{R}(\opt) - \mathcal{R}(G_{i-1}) &\le \mathcal{R}(\opt \cup G_{i-1}) - \mathcal{R}(G_{i-1})  \\
    &= \mathcal{R}(\opt \setminus G_{i-1} \cup G_{i-1}) - \mathcal{R}(G_{i-1}).
\end{align*}
Assume \(\opt\setminus G_{i-1} = \{y_1, \ldots, y_m\}\), where \(y_i\) are vertices in \(\opt\) but not in solution \(G_{i-1}\). Let us now define for any \(j=1, \ldots, m\) the marginal reward for any element \(y_j\) as \(Z_j = \mathcal{R}(G_{i-1} \cup \{y_1, \ldots, y_j\}) - \mathcal{R}(G_{i-1} \cup \{y_1, \ldots, y_{j-1}\})\).
Thus,
\begin{multline*}
    \mathcal{R}(\opt) - \mathcal{R}(G_{i-1}) \le \mathcal{R}(\opt \setminus G_{i-1} \cup G_{i-1}) - \mathcal{R}(G_{i-1}) = \\
    \mathcal{R}(G_{i-1} \cup \{y_1, \ldots, y_m\}) - \mathcal{R}(G_{i-1}) = \\
    \mathcal{R}(G_{i-1} \cup \{y_1, \ldots, y_m\}) - \mathcal{R}(G_{i-1} \cup \{y_1, \ldots, y_{m-1}\}) + \\
    \mathcal{R}(G_{i-1} \cup \{y_1, \ldots, y_{m-1}\}) - \\
    \mathcal{R}(G_{i-1} \cup \{y_1, \ldots, y_{m-2}\}) + \ldots - \mathcal{R}(G_{i-1}) = \sum_{j=1}^m Z_j.
\end{multline*}
Notice that, for any $j$,
\begin{align}
    \label{eq:bound Z}
    \frac{Z_j}{\mathcal{C}(G_{i-1} \cup y_j) - \mathcal{C}(G_{i-1})} &\le \frac{\mathcal{R}(G_{i-1} \cup y_j) - \mathcal{R}(G_{i-1})}{\mathcal{C}(G_{i-1} \cup y_j) - \mathcal{C}(G_{i-1})} \nonumber \\ 
    &\le \frac{\mathcal{R}(G_{i}) - \mathcal{R}(G_{i-1})}{\mathcal{C}(G_{i}) - \mathcal{C}(G_{i-1})}
\end{align}
where the first inequality holds due to the modularity of $\mathcal{R}$ (precisely, Eq.~\eqref{eq:sumbodularity def 1} where $Y=G_{i-1} \cup \{y_1, \ldots, y_j\}$ and $X=G_{i-1} \cup y_j$ and $X \subseteq Y$ ) and the second inequality holds due to the greedy rule, according to
Eq.~\eqref{eq:passo greedy}. 
Therefore,
\begin{align*}
    \mathcal{R}(\opt) &- \mathcal{R}(G_{i-1}) \le \sum_{j=1}^m Z_j \\
    &\le \frac{\mathcal{R}(G_{i}) - \mathcal{R}(G_{i-1})}{\mathcal{C}(G_{i}) - \mathcal{C}(G_{i-1})} \sum_{j=1}^m \mathcal{C}(G_{i-1} \cup y_j) - \mathcal{C}(G_{i-1}) \\
    &\le \frac{\mathcal{R}(G_{i}) - \mathcal{R}(G_{i-1})}{\mathcal{C}(G_{i}) - \mathcal{C}(G_{i-1})} \sum_{y \in \opt} \mathcal{C}(y).
\end{align*}

Since function \(c\) can be computed in polynomial time, we can bound the sum of the costs over all the elements with the cost of the optimal solution that at its turn cannot exceed $B$. Hence,
\pushQED{\qed} 
\begin{align*}
    &\mathcal{R}(\opt) - \mathcal{R}(G_{i-1}) \le \frac{\mathcal{R}(G_{i}) - \mathcal{R}(G_{i-1})}{\mathcal{C}(G_{i}) - \mathcal{C}(G_{i-1})} \sum_{y \in \opt} \mathcal{C}(y) \\
    &\le \mathcal{C}(\opt) \frac{\mathcal{R}(G_{i}) - \mathcal{R}(G_{i-1})}{\mathcal{C}(G_{i}) - \mathcal{C}(G_{i-1})} \le B \frac{\mathcal{R}(G_{i}) - \mathcal{R}(G_{i-1})}{\mathcal{C}(G_{i}) - \mathcal{C}(G_{i-1})}. \qedhere
\end{align*}
\popQED
\let\qed\relax
\end{proof}

In the following, Lemma~\ref{lemma:distance from optimum} proves that the solution of the greedy algorithm at each iteration can be bounded with respect to the optimal solution.
\begin{lemma}\label{lemma:distance from optimum}
For any \(i = 1, \ldots, \ell+1\) it holds that
\[
    \mathcal{R}(G_i) \ge \left[1 - \prod_{k=1}^{i} \left( 1 - \frac{\mathcal{C}(G_k) - \mathcal{C}(G_{k-1})}{B} \right) \right] \mathcal{R}(\opt),
\]
where \(\opt\) is the optimal solution considering budget \(B\).
\end{lemma}
\begin{proof}
For \(i=1\), the proof follows directly from Lemma~\ref{lemma:marginal gain} since $G_0= \emptyset$, $\mathcal{R}(G_0)=0$ and $\mathcal{C}(G_0)=0$. 
For \(i>1\), using Lemma~\ref{lemma:marginal gain} and  inductive hypothesis on $\mathcal{R}(G_{i-1})$, it holds:
\pushQED{\qed} 
\begin{align*}
    \mathcal{R}(G_i) &= \mathcal{R}(G_{i-1}) - \mathcal{R}(G_{i-1}) + \mathcal{R}(G_i) 
    \\
    &\ge \mathcal{R}(G_{i-1}) + \frac{\mathcal{C}(G_i) - \mathcal{C}(G_{i-1})}{B} (\mathcal{R}(\opt) - \mathcal{R}(G_{i-1})) 
    \\
    &= \left( 1 - \frac{\mathcal{C}(G_i) - \mathcal{C}(G_{i-1})}{B} \right) \mathcal{R}(G_{i-1}) + 
    \\ 
    &\frac{\mathcal{C}(G_i) - \mathcal{C}(G_{i-1})}{B} \mathcal{R}(\opt) 
    \\
    &\ge \left( 1 - \frac{\mathcal{C}(G_i) - \mathcal{C}(G_{i-1})}{B} \right) \cdot
    \\
    &\left[ 1 - \prod_{k=1}^{i-1} \left( 1 - \frac{\mathcal{C}(G_k) - \mathcal{C}(G_{k-1})}{B} \right)\right] \mathcal{R}(\opt) +
    \\
    &\frac{\mathcal{C}(G_i) - \mathcal{C}(G_{i-1})}{B} \mathcal{R}(\opt) 
    \\
    &= \left( 1 - \prod_{k=1}^{i} \left( 1 - \frac{\mathcal{C}(G_k) - \mathcal{C}(G_{k-1})}{B} \right) \right) \mathcal{R}(\opt). \qedhere 
\end{align*}

\popQED
\let\qed\relax
\end{proof}

Finally, exploiting Lemma~\ref{lemma:distance from optimum}, we prove that \apxmre  provides a constant approximation ratio. 

%
\begin{theorem}\label{thm-approx}
Algorithm \apxmre guarantees a \( \frac{1}{2} \left( 1 - \frac{1}{e} \right) \) approximation ratio for \prob.
\end{theorem}
\begin{proof}
By applying Lemma~\ref{lemma:distance from optimum}, by the fact that \(\mathcal{C}(G_{\ell+1}) > B\) (since it violates the constraint), and recalling that for $a_1 , \ldots, a_n \in \mathbb{R}$ such that $\sum_{i=1}^n a_i = A$, the function $(1 - \Pi_{i=1}^{n} (1 - \frac{a_i}{A}))$ achieves its minimum at $a_1 = \ldots = a_n = \frac{A}{n}$, we have:
\begin{multline*}
    \mathcal{R}(G_{\ell+1}) \ge \left[ 1 - \prod_{k=1}^{\ell+1} \left( 1 - \frac{\mathcal{C}(G_k) - \mathcal{C}(G_{k-1})}{B} \right) \right] \mathcal{R}(\opt)
    \\
     \ge \left[ 1 - \prod_{k=1}^{\ell+1} \left( 1 - \frac{\mathcal{C}(G_k) - \mathcal{C}(G_{k-1})}{\mathcal{C}(G_{\ell+1})} \right) \right] \mathcal{R}(\opt)
    \\
     \ge \left[ 1 - \left( 1 - \frac{1}{(\ell+1)} \right)^{\ell+1} \right] \mathcal{R}(\opt) 
\ge \left( 1 - \frac{1}{e} \right) \mathcal{R}(\opt).
\end{multline*}
Since \(r\) is a modular function we know that the marginal gain selecting element \(x_{\ell+1}\) at step \(\ell+1\) is less than or equal to the gain given by a solution that contains the element with maximum gain, i.e., \(\mathcal{R}(G_{\ell+1}) - \mathcal{R}(G_\ell) = \mathcal{R}(x_{\ell+1}) \le \mathcal{R}(x^*)\), where \(\mathcal{R}(x_{\ell+1})\) is the marginal gain obtained adding the $(\ell+1)$-th vertex selected by greedy rule and \(x^* = \arg\max_{x\in V : \mathcal{C}(x)\le B} \mathcal{R}(x)\), (i.e., $x^*$ is the vertex with the maximum reward reachable with budget $B$), it follows:
\[
    \mathcal{R}(G_\ell) + \mathcal{R}(x^*) \ge \mathcal{R}(G_{\ell+1}) \ge \left( 1 - \frac{1}{e} \right) \mathcal{R}(\opt), \text{ and }
\]
\pushQED{\qed} 
\[
    \max\{ \mathcal{R}(G_\ell), \mathcal{R}(x^*) \} \ge \frac{1}{2} \left( 1 - \frac{1}{e} \right) \mathcal{R}(\opt). \qedhere
\]
\let\qed\relax
\popQED
\end{proof}

In Algorithm~\ref{alg:apxmre}, $x^*=S_1$ and $\mathcal{R}(S_1)=\mathcal{R}(x^*)$. Moreover, $\mathcal{R}(S_2) \ge \mathcal{R}(G_\ell)$ because $S_2$ contains not only the vertex selected by the greedy rule $\{x_1, \ldots, x_{\ell}\}$, but also the vertices on its left on the same row (see Algorithm~\ref{alg:apxmre}, Line~\ref{code:3-updates}).

It is worth mentioning that the same approximation ratio also holds for \apxmrc.
Lemma~\ref{lemma:marginal gain} holds even when the greedy algorithm picks a subset of vertices in a row.
\apxmrc considers the cumulative reward of a partial row, that is not already in the solution, up to the vertex in consideration. 
The cumulative reward \(\mathcal{R}(G_i \cup y_j)\) for any \(j\) is greater than or equal to that obtained by picking only the last element $y_j$, whereas the cost \(\mathcal{C}(G_i \cup y_j)\) remains the same since the cost for visiting only the last element $y_j$ or all the elements up to $y_j$, does not change. 
Therefore, the second inequality in Eq.~\eqref{eq:bound Z} of Lemma~\ref{lemma:marginal gain}  remains true even if \(G_i \setminus G_{i-1}\) is a subset of vertices in a row of the graph.
So, repeating the reasoning of Lemma~\ref{lemma:distance from optimum} and Theorem~\ref{thm-approx}, we prove that algorithm  \apxmrc guarantees the same approximation ratio as \apxmre.

\section{Performance Evaluation}\label{sec:simulations}
We evaluate the performance of the new algorithms $\{\gdyme, \gdymc, \apxmre, \apxmrc\}$, and compare them with the optimal algorithm \optsc.
For each scenario, each algorithm is tested with an increasing budget $B=2n, \ldots, 2(mn+m)$,
and we plot the average of the results on $30$ instances ($10$ in the case of real data) along with their $95\%$ confidence interval.
Specifically, in Figures~\ref{fig:experiments} and~\ref{fig:experiments-real}, organized in paired plots, we report:
\begin{itemize}
    \item on the left plot of each pair, the collected reward $\mathcal{R}$ in percentage ($100\%$ means all the available rewards have been collected), with respect to the used budget $B$ in percentage ($100\%$ means $B=2(mn+m)$);
    \item on the right plot of each pair, the ratio $\rho = \frac{\mathcal{R}(\alg)}{\mathcal{R}(\optsc)}$ between the reward of $\alg~=~\{$\gdyme, \gdymc, \apxmre, \apxmrc$\}$ and the reward of \optsc.
    Recall $\optsc$ optimally solves \prob.
\end{itemize}

\begin{figure}[ht]
	\centering
	\subfloat[$\theta = 0$.]{%
		\includegraphics[height=2.85cm]{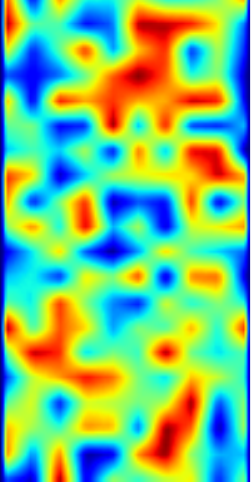}
		\label{img:plot_heatmap_100}
	}
	\subfloat[$\theta = 0.9$.]{%
		\includegraphics[height=2.85cm]{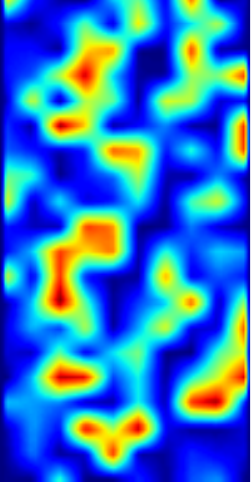}
		\label{img:plot_heatmap_300}
	}
	\subfloat[$\theta = 1.8$.]{%
		\includegraphics[height=2.85cm]{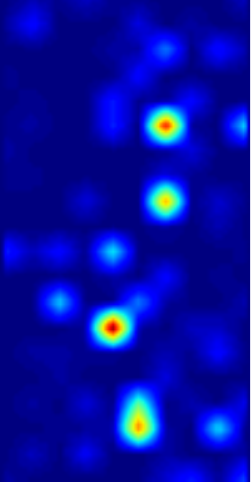}
		\label{img:plot_heatmap_500}
	}
	\subfloat[$\theta = 2.7$.]{%
		\includegraphics[height=2.85cm]{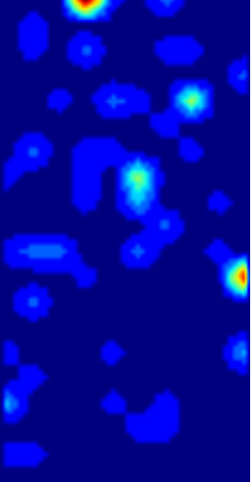}
		\label{img:plot_heatmap_700}
	}
	\subfloat[Real.]{%
		\includegraphics[height=2.85cm]{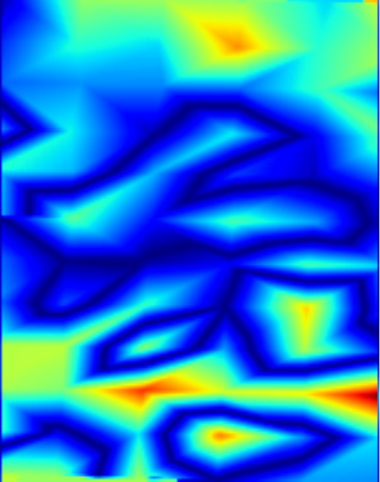}
		\label{img:plot_heatmap_1010}
	}
	\hfill
	\subfloat[Legend.]{%
		\includegraphics[width=5cm]{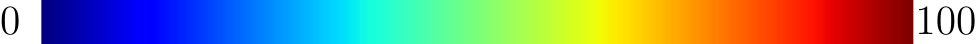}
		\label{img:plot_heatmap_legend}
	}
	\caption{Reward maps for synthetic (a)--(d), and real (e) graph instances. Synthetic graphs are $A(100,50)$ while real graph is $A(274,214)$.
	 \revision{The heat-maps, rendered after a \emph{bilinear interpolation} using matplotlib software, highlight the areas with large and small rewards with hot and cold colors, respectively (see legend (f)).
	 Synthetic data are randomly generated, while real data are taken from real observations~\cite{thayer2018multi}.
	 In particular, real moisture levels are then converted in actual rewards through a suitable interpolation~\cite{oliver1990kriging}.}}
	\label{img:heatmaps}
\end{figure}


\begin{figure*}[ht]
    \captionsetup[subfloat]{farskip=-3pt,captionskip=-7pt}
	\centering
	\subfloat[$A(100,50), \theta = 0$.]{%
		\includegraphics[scale=0.90]{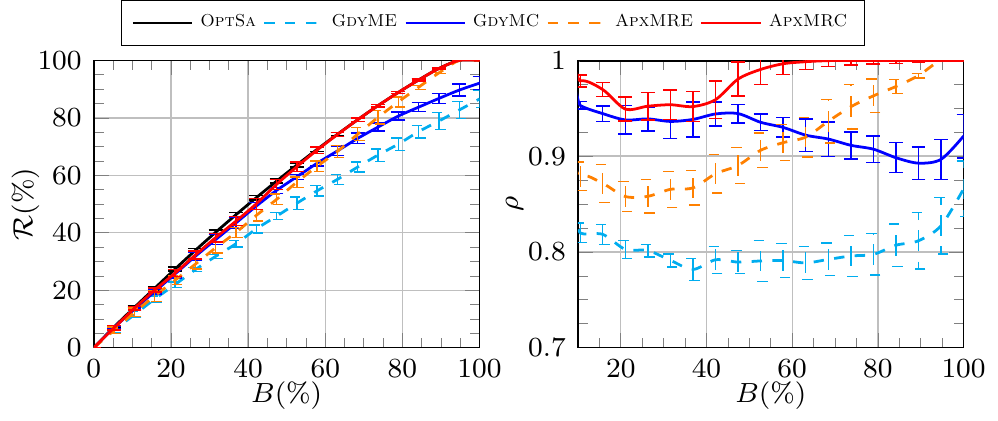}
		\label{fig:plot_compare_100}
	}
	\subfloat[$A(50,100), \theta = 0$.]{%
		\includegraphics[scale=0.90]{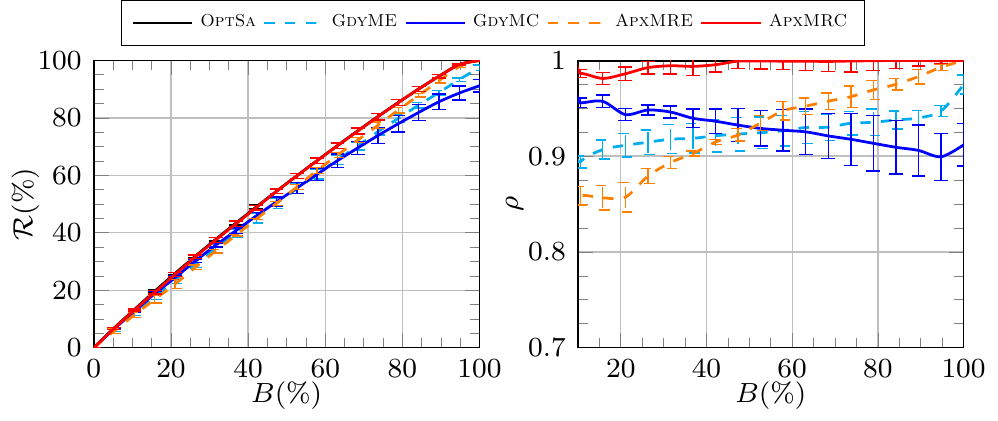}
		\label{fig:plot_compare_200}
	}
 	\hfill
	\subfloat[$A(100,50), \theta = 0.8$.]{%
		\includegraphics[scale=0.90]{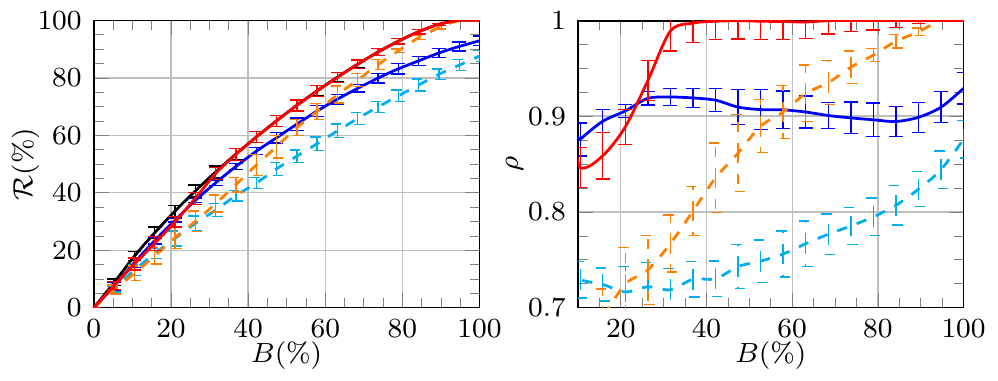}
		\label{fig:plot_compare_300}
	}
	\subfloat[$A(50,100), \theta = 0.8$.]{%
		\includegraphics[scale=0.90]{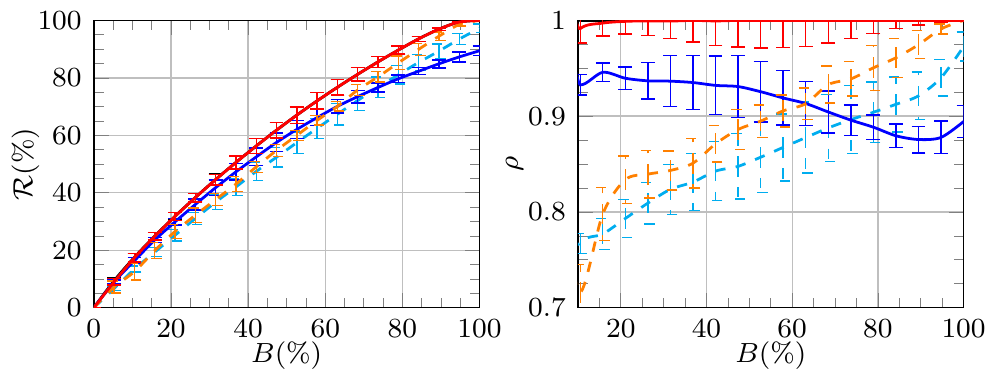}
		\label{fig:plot_compare_400}
	}
	\hfill
	\subfloat[$A(100,50), \theta = 1.9$.]{%
		\includegraphics[scale=0.90]{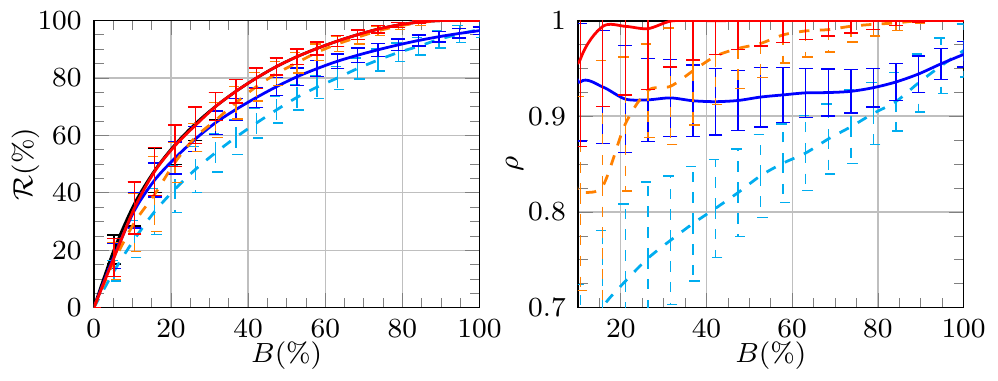}
		\label{fig:plot_compare_500}
	}
	\subfloat[$A(50,100), \theta = 1.9$.]{%
		\includegraphics[scale=0.90]{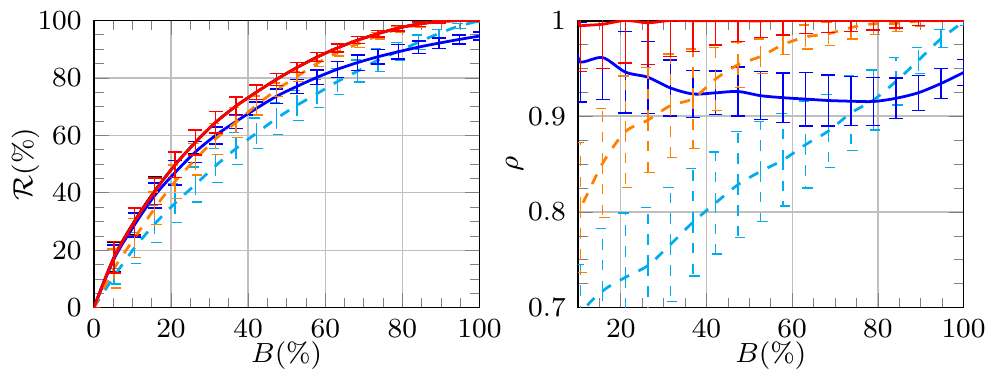}
		\label{fig:plot_comparen_600}
	}
 	\hfill
	\subfloat[$A(100,50), \theta = 2.7$.]{%
		\includegraphics[scale=0.90]{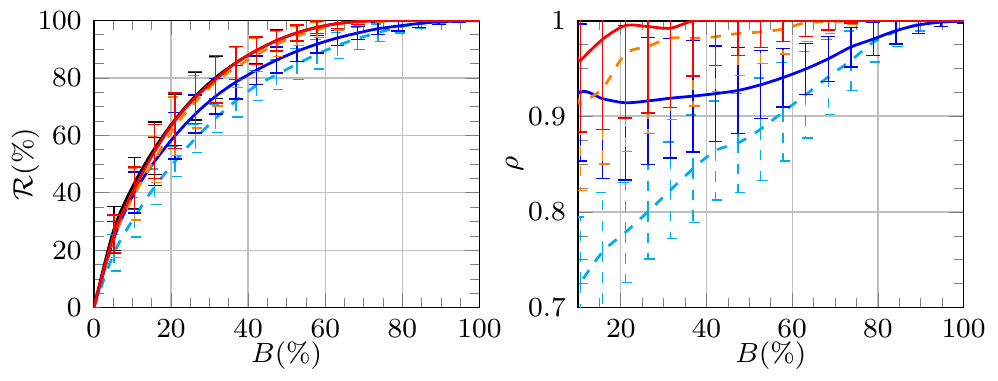}
		\label{fig:plot_compare_700}
	}
	\subfloat[$A(50,100), \theta = 2.7$.]{%
		\includegraphics[scale=0.90]{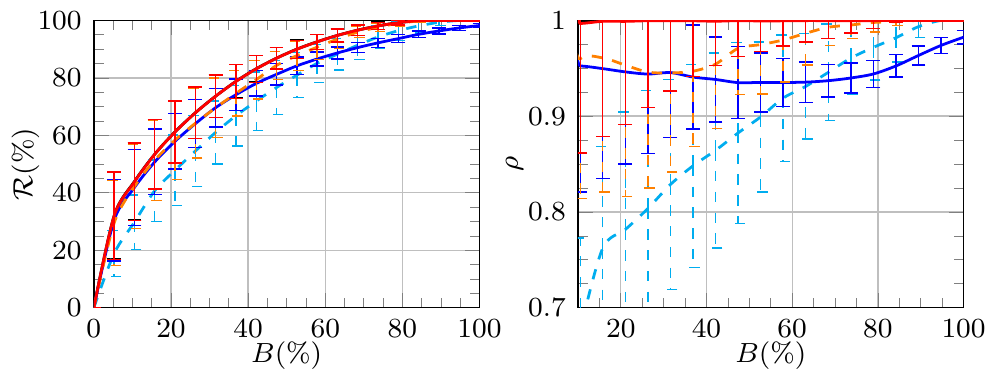}
		\label{fig:plot_compare_800}
	}
	\caption{Performance comparison of the proposed \gdyme, \gdymc, \apxmre, \apxmrc, and \optsc algorithms.}
	\label{fig:experiments}
\end{figure*}

\subsection{On Synthetic Data}
In this section, we evaluate the performance of the presented algorithms on generated synthetic data.
We assume the rewards are properly randomized following the Zipf distribution~\cite{tullo2003modelling}, which is characterized by a single parameter $\theta$ that rules the occurrences of the reward values.
In our evaluations, the rewards are integers in the interval $[0,100)$.
When $\theta=0$, the rewards are uniformly distributed in $[0,100)$, while when $\theta$ increases, the smallest rewards become more and more frequent than the largest ones.
In our setting, we assume $\theta=\{0,0.9,1.8,2.7\}$.

Examples of random reward maps are illustrated in Figures~\ref{img:plot_heatmap_100}--\ref{img:plot_heatmap_700}, where the hot and cold colors represent high and low rewards, respectively.
In our applications mentioned at the beginning, a hot color represents a place where the robot has to operate on a plant or pick an item from the shelves with a high reward (high priority), while a cold color represents a location
with a low reward (low priority).
When $\theta=0$ (Figure~\ref{img:plot_heatmap_100}), the heat map illustrates rewards uniformly distributed, while when $\theta=1.8$  (Figure~\ref{img:plot_heatmap_500}), the heat map features highly unbalanced rewards, with very dispersed positions with high reward.
In our experiments, we finally vary the size of the aisle-graphs considering two different layouts: more rows than columns ($m > n$, $A(100,50)$) and more columns than rows ($n > m$, $A(50,100)$).

For each $A$ and $\theta$, we run the five algorithms.
Fixing $m$, $n$, and $\theta$, we generate $30$ different reward distributions and solve them with each algorithm.
For each random distribution, each algorithm is tested with an increasing budget.
Since the overall reward is randomly generated and differs for each graph, we return the reward gained as a percentage of the overall reward of each graph. 
Finally, we plot the average of the results along with their $95\%$ confidence interval.

Figure~\ref{fig:experiments} compares the reward collected by the different algorithms for several values of $\theta$. 
When $\theta = 0$, the percentage of the collected reward almost increases linearly with the budget; this is reasonable since the rewards are uniformly distributed, and so also the gain for unit of cost is uniformly distributed.
There are no subareas of the map that are more important to visit because their reward per unit is larger.

When $\theta \ge 1.9$, there are instead localized subareas with higher rewards.
In such cases, 
the robot prioritizes those areas (hot areas of the heat maps) and collects a high percentage of reward (up to $60\%$ when $\theta=1.9$ and up to $80\%$ when $\theta=2.7$) already with a small budget $B \approx 20\%$.
Up to $B \approx 60\%$, the value of $\theta$ makes the difference in the percentage of the collected reward.
With $B \approx 80\%$, instead, $\theta$ does not meaningfully impact  the reward.
This means that when $B \ge 80\%$, there is enough budget to cover the most meaningful part of the map in any case, and it is not important
to first select  the more promising paths.

When $\theta \le 0.8$, the  collected reward
highly depends on the algorithms for any value of $B \le 60\%$,
showing that the strategy followed by the algorithms is important.
In principle, there are no paths that are definitely more promising than others and it is difficult to decide which one is best.
When $m>n$, the large number $m$ of rows makes the decision even harder.
Finally, 
all the algorithms report more or less the same rewards (all the curves are near to each other), independently of the $\theta$ value when $B\ge 80\%$. In this case, the larger budget   balances the algorithm cleverness. 

To confirm that the problem difficulty increases with the value of $m$,
observe that
the best performing Algorithm \apxmrc performs better when $m < n$ than when $m > n$.
The collected reward by \apxmrc (solid red line) almost coincides with that of the optimal Algorithm \optsc (solid black line) for the scenarios with $m < n$.
In fact, the ratio $\rho=\frac{\mathcal{R}(\apxmrc)}{\mathcal{R}(\optsc)} \rightarrow 1$ when $m < n$ for any value of $\theta$ (e.g., Figure~\ref{fig:plot_compare_200}), while when $m > n$ the ratio $\rho \rightarrow 1$ only for higher values of $\theta$ (e.g., Figure~\ref{fig:plot_compare_700}). When $\theta$ is small and $m$ is large, \apxmrc does not find the optimal solution,
which instead is computed by \optsc.
Note that  \apxmrc always outperforms  \apxmre (dashed red line).
The ratio  between \apxmre and \optsc  is well above $0.9$ (but below the ratio of \apxmrc) when $B \ge 40\%$,
and around $0.8$ otherwise.
This is because \apxmre greedily adds the best vertex $v_{i,j}$ having the maximum ratio between its reward and the cost for reaching it, 
without including the reward of the vertices it passes through in its selection criterion. Instead, \apxmrc considers also the reward of such vertices for
the selection criterion. 
Similarly, the reward collected from \gdymc (solid blue line) is greater than that of \gdyme (dashed blue line). 
In all plots, the ratio $\rho$ between \gdymc and \optsc is above $0.9$.
Instead, the ratio  between \gdyme and \optsc increases with $B$ from $0.7$ to $0.9$, and more. When $B$ is small, \gdyme  performs much worse than \gdymc because it can spend a lot of budget to cover isolated high reward items. When $B \ge 80\%$ and $m<n$, \gdyme 
outperforms \gdymc probably because the single reward becomes the most important selection criterion to
distinguish among important and less important items.

In almost all the plots, up to $B=40\%$, \gdymc outperforms \apxmre, that is, the cumulative reward
beats the criteria that consider the cost. It seems that \gdymc gives the precedence to the cluster of items
with high reward independently of how far they are, while \apxmre can prefer the closest item even if is has small reward.
Namely, a small cost at the denominator of the selection criterion amplifies minimal rewards at the numerator for \apxmre.
This behavior of \gdymc is particularly important because \gdymc is computationally more efficient in time and space than \apxmre and \apxmrc.
So when the budget is limited and the computational efficiency matters, \gdymc 
is the best algorithm to use.  

Lastly, it can be experimentally seen that the reward achieved by \apxmre and \apxmrc is much larger than the guaranteed
reward. 
In our experiments, both \apxmre and \apxmrc  collect more than $90\%$ of the optimum reward, i.e., 
the ratio $\rho \ge 0.9$, while the guaranteed approximation ratio 
proved in Section~\ref{sec:approximation} is $\approx 0.32$.

\subsection{On Real Data}
To further asses the effectiveness of the proposed algorithms, in this section we evaluate their performance on a set of instances 
coming from real data in a vineyard irrigation scenario.
These data sets were obtained
from a large scale commercial vineyard located in central California~\cite{thayer2018multi}.
The section we consider has $274$ rows and $214$ columns (i.e., aisle-graph $A(274, 214)$), resulting in approximately $60,000$ vines in the aisle-graph.
For each (internal) vertex $v \in V$, the reward is set to $\mathcal{R}(v) = |T - m(v)|$, where $T$ is a constant indicating the desired soil moisture in the vineyard (provided by a human expert), and $m(v)$ is the soil moisture at vertex $v$. 
This reward is the difference between desired moisture and actual moisture, thus revealing how underwatered or overwatered a vine is. 
Due to the large size of the ranch, soil moisture values were 
manually sampled at discrete locations  using a probe
equipped with a GPS.
From the finite set of samples, a soil moisture map for the whole block was obtained using the  Kriging algorithm  for interpolation~\cite{oliver1990kriging}.

Figure~\ref{img:plot_heatmap_1010} shows one of the soil moisture maps used in our experiments on real data.
In particular, data were collected every two weeks from this vineyard and used to produce ten soil moisture maps. 
These  maps (e.g., Figure~\ref{img:plot_heatmap_1010}) were then used to test the proposed algorithms, and the results were averaged across each reward map.
For reasons of similarity with the synthetic data, we have also tested our results on the transposed soil moisture maps of size $214$ rows and $274$ columns (i.e., aisle-graph $A(214, 274)$) with $m < n$.
This is reasonable since the structure of the constrained aisle graph and the soil moisture are independent of each other.

For each of the ten scenarios, we run \gdyme, \gdymc, \apxmre, \apxmrc, and \optsc, with an increasing budget.
Finally, we plot the average of the results along with their confidence interval.

Figure~\ref{fig:experiments-real} compares the reward for different algorithms on the real ten soil moisture maps.
The plots in Figure~\ref{fig:experiments-real} on the left  matches the trend already seen in the plots in Figure~\ref{fig:experiments} (synthetic data) 
when $\theta=0.8$ because the collected reward $\mathcal{R}$ when $B=20\%$ is slightly above $20\%$, when $B=40\%$ is above $60\%$, and $B=80\%$ is slightly above $80\%$. This is not surprising because in the literature Zipf distributions with parameter $\theta=0.8$ model several real distributions.

In the right plots of Figure~\ref{fig:experiments-real}, the most interesting difference between the synthetic and the real data happens for algorithm \gdyme (dashed blue line)and \gdymc (solid blue line). 
In the synthetic data, \gdymc outperforms \gdyme; while in real data,  \gdymc and \gdyme perform the same.
This means that \gdyme, which takes into account for selection only the single maximum reward but adds the
rewards of all traversed vertices, performs the same selection as \gdymc, which takes into account for selection the cumulative reward. In
this case it appears that the cumulative reward is dominated by the largest reward.

The ratio between \apxmrc and \optsc tends to $1$ already when $B=20\%$ for both cases $m>n$ and $m<n$. The performance of  the case $m<n$ is more similar to that of the case $m>n$ in the real scenario than in the synthetic one. This is arguably due to the fact that  $m$ and $n$ are much closer in relative magnitude, i.e., $\frac{274}{214}=1.28$, than $100$ and $50$ whose ratio is $\frac{100}{50}=2$.
The \apxmre algorithm poorly performs when $B \le 40\%$, as it happens for the synthetic data only when $m >n$. In contrast, when $m<n$, \apxmre is always above \gdymc, which is the same as \gdyme.
This can be explained imagining the transpose of
Figure~\ref{img:plot_heatmap_1010}: all the rows look quite similar among them and so
it is more profitable to collect the reward from the closest rows.
In conclusion, the performance of the proposed algorithms on the real scenario is well above the performance guaranteed for every percentage of $B$.

\begin{figure}[ht]
    \captionsetup[subfloat]{farskip=-3pt,captionskip=-7pt}
	\centering
	\subfloat[$A(274,214)$.]{%
		\includegraphics[scale=0.90]{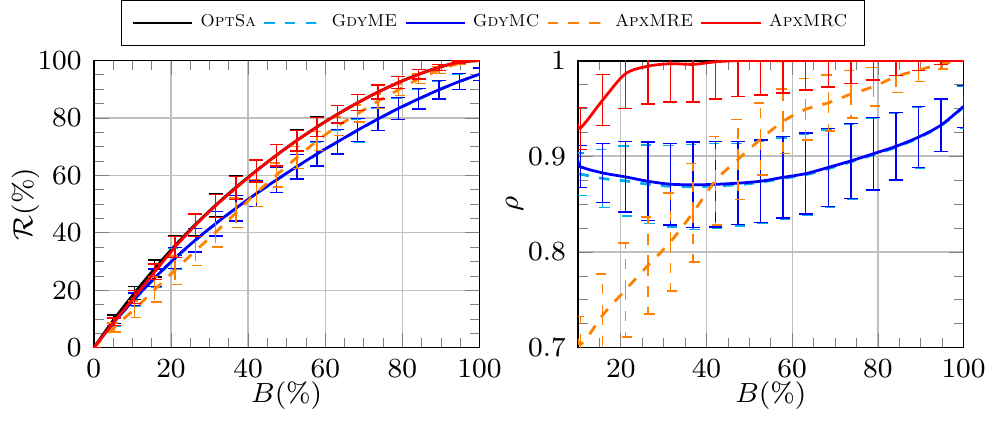}
		\label{fig:plot_compare_1010}
	}
	\hfill
	\subfloat[$A(214,274)$.]{%
		\includegraphics[scale=0.90]{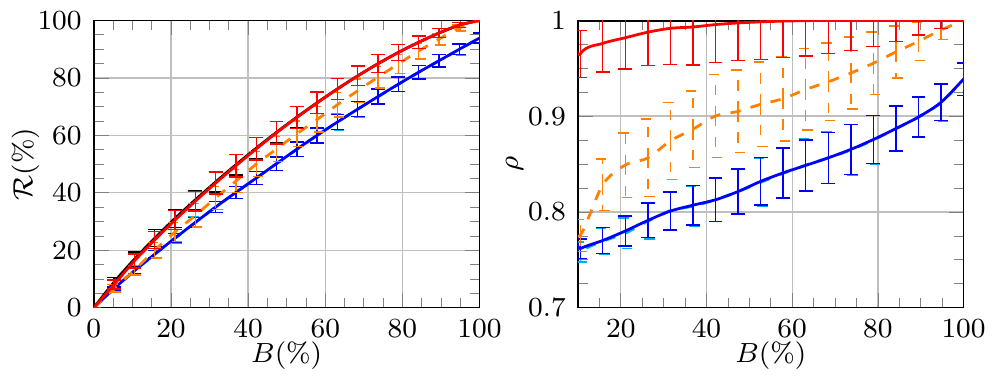}
		\label{fig:plot_compare_1110}
	}
	\caption{Performance comparison of the proposed algorithms on real graphs.}
	\label{fig:experiments-real}
\end{figure}

\section{Conclusions}\label{sec:conclusions}
In this paper, we focused on \prob, a route planning problem in constrained aisle-graphs with single access subject to a given budget, and aiming to maximize an assigned reward.
We first proposed a polynomial-time algorithm to optimally solve \prob.
This solution, although optimal, has a  computational cost that can become problematic when the size of the aisle-graph grows.
Thus, we proposed four simpler and faster greedy algorithms with reduced computational cost in time and space.
For two of them, we exploit submodularity properties and achieve a provably guaranteed approximation ratio of \(\frac{1}{2}(1 - \frac{1}{e})\).
Finally, we extensively evaluated the new four algorithms with respect to the optimal solution on synthetic and real data.
As future work, it would be worth to apply similar approaches to different topologies for the input graph. 
Different applications as well as different robot capabilities may lead to interesting scenarios.
We also extend our investigation to the case where swarms of robots can jointly accomplish tasks.


\bibliographystyle{IEEEtran}
\bibliography{IEEEabrv,main-t-ro-20_revised}

\begin{IEEEbiography}
[{\includegraphics[width=1in,height=1.25in,clip,keepaspectratio]{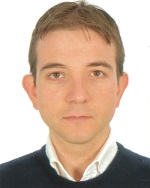}}] {Francesco Betti Sorbelli}
received the Bachelor and Master degrees {\em cum laude} in Computer Science from the University of Perugia, Italy, in 2007 and 2010, respectively, 
and his Ph.D. in Computer Science from the University of Florence, Italy, in 2018.
He was a postdoc researcher at University of Perugia in 2018 under the supervision of Prof. Cristina M. Pinotti.
Currently, he is a postdoc at the Missouri University of Science and Technology University, USA, under the supervision of Prof. Sajal K. Das.
His research interests include algorithms design, combinatorial optimization, unmanned vehicles.
\end{IEEEbiography}

\begin{IEEEbiography} [{\includegraphics[width=1in,height=1.25in,clip,keepaspectratio]{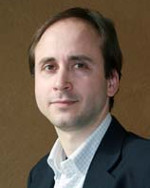}}] {Stefano Carpin} is Professor in the Department of Computer Science and 
Engineering at the University of California, Merced. He received his \emph{Laurea} and Ph.D. degrees
in electrical engineering and computer science from the University of
Padova, Italy in 1999 and 2003, respectively. From 2003 to 2006 he
held faculty positions with Jacobs University Bremen, Germany. Since
2007 he has been with the School of Engineering at UC Merced, where he
established and leads the robotics laboratory. His research interests
include mobile and cooperative robotics for service tasks, and robot
algorithms. He served as associated editor for the IEEE Transactions on
Automation Science and Engineering and for the IEEE Transactions
in Robotics and he is the founding chair for the Department of Computer
Science and Engineering at the University of California, Merced.
\end{IEEEbiography}

\begin{IEEEbiography} [{\includegraphics[width=1in,height=1.25in,clip,keepaspectratio]{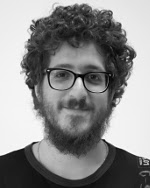}}] {Federico Corò}
received the Bachelor and Master degrees {\em cum laude} in Computer Science from the University of Perugia, Italy, in 2014 and 2016, respectively and his Ph.D. in Computer Science in 2019 at Gran Sasso Science Institute (GSSI), L'Aquila, Italy.
Currently, he is a postdoc researcher in the Department of Computer Science at La Sapienza in Rome, Italy.
His research interests include several aspects of theoretical computer science, including combinatorial optimization, network analysis, and the design and efficient implementation of algorithms.
\end{IEEEbiography}

\begin{IEEEbiography} [{\includegraphics[width=1in,height=1.25in,clip,keepaspectratio]{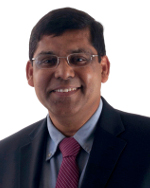}}] {Sajal K. Das} is a professor of computer science and Daniel St. Clair Endowed Chair at Missouri University of Science and Technology. 
His research interests include wireless sensor networks, mobile and pervasive computing, cyber-physical systems and IoT, smart environments, cloud computing, cyber security, and social networks. He serves as the founding Editor-in-Chief of Elsevier's Pervasive and Mobile Computing journal, and as Associate Editor of several journals including the IEEE Transactions of Mobile Computing, IEEE Transactions on Dependable and Secure Computing, and ACM Transactions on Sensor Networks. He is an IEEE Fellow.
\end{IEEEbiography}

\begin{IEEEbiography} [{\includegraphics[width=1in,height=1.25in,clip,keepaspectratio]{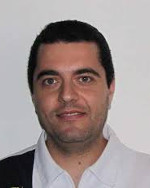}}] {Alfredo Navarra}
is Associate Professor since 2015 at the Mathematics and 
Computer Science Dept, University of Perugia, Italy. He received his Ph.D. 
in Computer Science in 2004 from ``Sapienza'' University of Rome. 
Before joining the University of Perugia in 2007, 
he has been with various international research institutes like the INRIA  
of Sophia Antipolis, France; the Dept of Computer Science at the Univ. 
of L'Aquila, Italy; the LaBRI, Univ. of Bordeaux, France.
His research interests include algorithms, computational complexity, 
distributed computing and networking.
\end{IEEEbiography}

\vspace{-0.3in}
\begin{IEEEbiography} [{\includegraphics[width=1in,height=1.25in,clip,keepaspectratio]{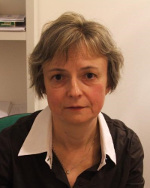}}] {Cristina M. Pinotti}
received the Master degree {\em cum laude} in Computer Science from the University of Pisa, Italy, in 1986.
In 1987-1999, she was Researcher with the National Council of Research in Pisa.
In 2000-2003, she was Associate Professor at the University  of Trento.
Since 2004, she is a Full Professor at the University of Perugia.
Her current research interests include the design and analysis of algorithms for wireless sensor networks and communication networks.
\end{IEEEbiography}

\end{document}

%% file: macro.tex
\newtheorem{theorem}{Theorem}

\newtheorem{lemma}{Lemma}

\newtheorem{definition}{Definition}

\newtheorem{problem}{Problem}

\newcommand{\descr}[1]{\medskip\noindent\textbf{#1}.}
\renewcommand{\paragraph}{\descr}

\newcommand{\R}{\mathbb{R}}

\newcommand{\revision}[1]{{\color{black}{#1}}}

\DeclareMathOperator*{\argmax}{arg\,max}

\newcommand{\imax}{\ensuremath{i\textsubscript{max}}\xspace}

\newcommand{\ofri}{\textsc{OFr-I}\xspace}
\newcommand{\gfr}{\textsc{GFr}\xspace}
\newcommand{\gpr}{\textsc{GPr}\xspace}
\newcommand{\hgc}{\textsc{HGc}\xspace}

\newcommand{\ofrilong}{{Optimal Full-Row Improved}\xspace}
\newcommand{\gfrlong}{{Greedy Full-Row}\xspace}
\newcommand{\gprlong}{{Greedy Partial-Row}\xspace}
\newcommand{\hgclong}{{Heuristic General-case}\xspace}

\newcommand{\optsc}{\textsc{OptSa}\xspace} 
\newcommand{\gdyme}{\textsc{GdyME}\xspace} 
\newcommand{\gdymc}{\textsc{GdyMC}\xspace} 
\newcommand{\apxmre}{\textsc{ApxMRE}\xspace} 
\newcommand{\apxmrc}{\textsc{ApxMRC}\xspace} 

\newcommand{\optsclong}{{Optimum Single Access}\xspace}
\newcommand{\gdymelong}{{Greedy Max Element}\xspace}
\newcommand{\gdymclong}{{Greedy Max Cumulative}\xspace}
\newcommand{\apxmrelong}{{Approximate Max Ratio Element}\xspace}
\newcommand{\apxmrclong}{{Approximate Max Ratio Cumulative}\xspace}

\newcommand{\prob}{OASP\xspace} 
\newcommand{\problong}{Orienteering Aisle-graphs Single-access Problem\xspace}

\newcommand{\opt}{\textsc{OPT}\xspace}
\newcommand{\alg}{\textsc{ALG}\xspace}